\documentclass[10pt,journal,compsoc]{IEEEtran}
\ifCLASSOPTIONcompsoc
  \usepackage[nocompress]{cite}
\else
  \usepackage{cite}
\fi
\usepackage{hyperref}       
\usepackage{url}            
\usepackage{booktabs} 
\usepackage{xcolor}
\usepackage{soul}
\usepackage{epsfig}
\usepackage{graphicx}
\usepackage{amsmath}
\usepackage{amssymb}
\usepackage{helvet}
\usepackage{courier}
\usepackage{algorithm}
\usepackage{makecell}
\usepackage{multirow}
\usepackage{url}
\usepackage{booktabs}
\usepackage{array}
\usepackage{paralist}
\usepackage{balance}
\usepackage{utfsym}    
\usepackage{graphicx}
\usepackage[numbers]{natbib}

\newtheorem{assumption}{Assumption}

\newcommand{\x}{{\bf x}}

\newcommand{\w}{{\bf w}}

\newcommand{\y}{{\bf y}}

\newcommand*{\Scale}[2][4]{\scalebox{#1}{$#2$}}

\newtheorem{proof}{Proof}
\usepackage{subcaption}
\usepackage{caption}

\usepackage{algorithm,algorithmicx,algpseudocode}

\newcommand{\myblue}[1]{\textcolor[RGB]{130,176,210}{#1}}
\newcommand{\myora}[1]{\textcolor[RGB]{255,190,122}{#1}}
\newcommand{\myred}[1]{\textcolor[RGB]{250,127,111}{#1}}

\newtheorem{thm}{Theorem}
\ifCLASSINFOpdf
\else
\fi
\begin{document}
%
\title{Learning to Rebalance Multi-Modal Optimization by Adaptively Masking Subnetworks}
%
%
%
%

\author{Yang Yang,
        Hongpeng Pan,
        Qing-Yuan Jiang,
        Yi Xu,
        and Jinghui Tang
\IEEEcompsocitemizethanks{\IEEEcompsocthanksitem Yang~Yang, Hongpeng Pan and Jinghui Tang are with the School of Computer Science and Engineering, Nanjing University of Science and Technology, Nanjing 210094, China. \protect\\
E-mail: yyang@njust.edu.cn
\IEEEcompsocthanksitem Yi~Xu is with the School of Control Science and Engineering, Dalian University of Technology, Dalian 116081, China. \protect\\
E-mail: yxu@dlut.edu.cn}
}

\IEEEtitleabstractindextext{%
\begin{abstract}
Multi-modal learning aims to enhance performance by unifying models from various modalities but often faces the ``modality imbalance'' problem in real data, leading to a bias towards dominant modalities and neglecting others, thereby limiting its overall effectiveness. To address this challenge, the core idea is to balance the optimization of each modality to achieve a joint optimum. Existing approaches often employ a modal-level control mechanism for adjusting the update of each modal parameter. However, such a global-wise updating mechanism ignores the different importance of each parameter. Inspired by subnetwork optimization, we explore a uniform sampling-based optimization strategy and find it more effective than global-wise updating. According to the findings, we further propose a novel importance sampling-based, element-wise joint optimization method, called \underline{A}daptively \underline{M}ask \underline{S}ubnetworks Considering Modal \underline{S}ignificance~(AMSS). Specifically, we incorporate mutual information rates to determine the modal significance and employ non-uniform adaptive sampling to select foreground subnetworks from each modality for parameter updates, thereby rebalancing multi-modal learning. Additionally, we demonstrate the reliability of the AMSS strategy through convergence analysis. Building upon theoretical insights, we further enhance the multi-modal mask subnetwork strategy using unbiased estimation, referred to as AMSS+. Extensive experiments reveal the superiority of our approach over comparison methods.
\end{abstract}

}

\maketitle

\IEEEdisplaynontitleabstractindextext

%
\IEEEpeerreviewmaketitle

\IEEEraisesectionheading{\section{Introduction}\label{sec:introduction}}
\IEEEPARstart{I}{n} the real world, object can always be characterized by multiple modalities. For example, in action recognition, one can integrate data from video, audio, and motion sensors to identify various human actions~\cite{sun2022human}. Similarly, in article classification, predictions can be made by comprehensively fusing both content and images~\cite{PereiraV14}. 
Compared with single-modal data, multi-modal data is more informative and covers a wider range of information dimensions and diversity. Hence, it is more important to use multiple modal data to perceive the world. By leveraging multi-modal data, multi-modal learning strives to surpass single-modal learning, capturing widespread attention across diverse domains~\cite{khan2021exploiting, lu2024fact,YanXSZT23}.  With the development of deep learning techniques~\cite{YangZZX019, kenton2019bert,YangSZFZXY23}, many multi-modal deep fusion networks have been proposed~\cite{YangWZL018,yang2019semi, yang2019comprehensive, YangYBZZGXY23, YangYBZZGXY23}. They often employ different joint training strategies such as feature interaction~\cite{ZhuLYZLQ23, joze2020mmtm, LiangQGCL22} (i.e., two modalities interact from the input level) and prediction ensemble~\cite{MaHZFZH21,NagraniYAJSS21} (i.e., two modalities interact from the prediction level), all while optimizing a unified learning objective.

\begin{figure}[!htbp]
  \centering
\includegraphics[width=0.9\columnwidth]{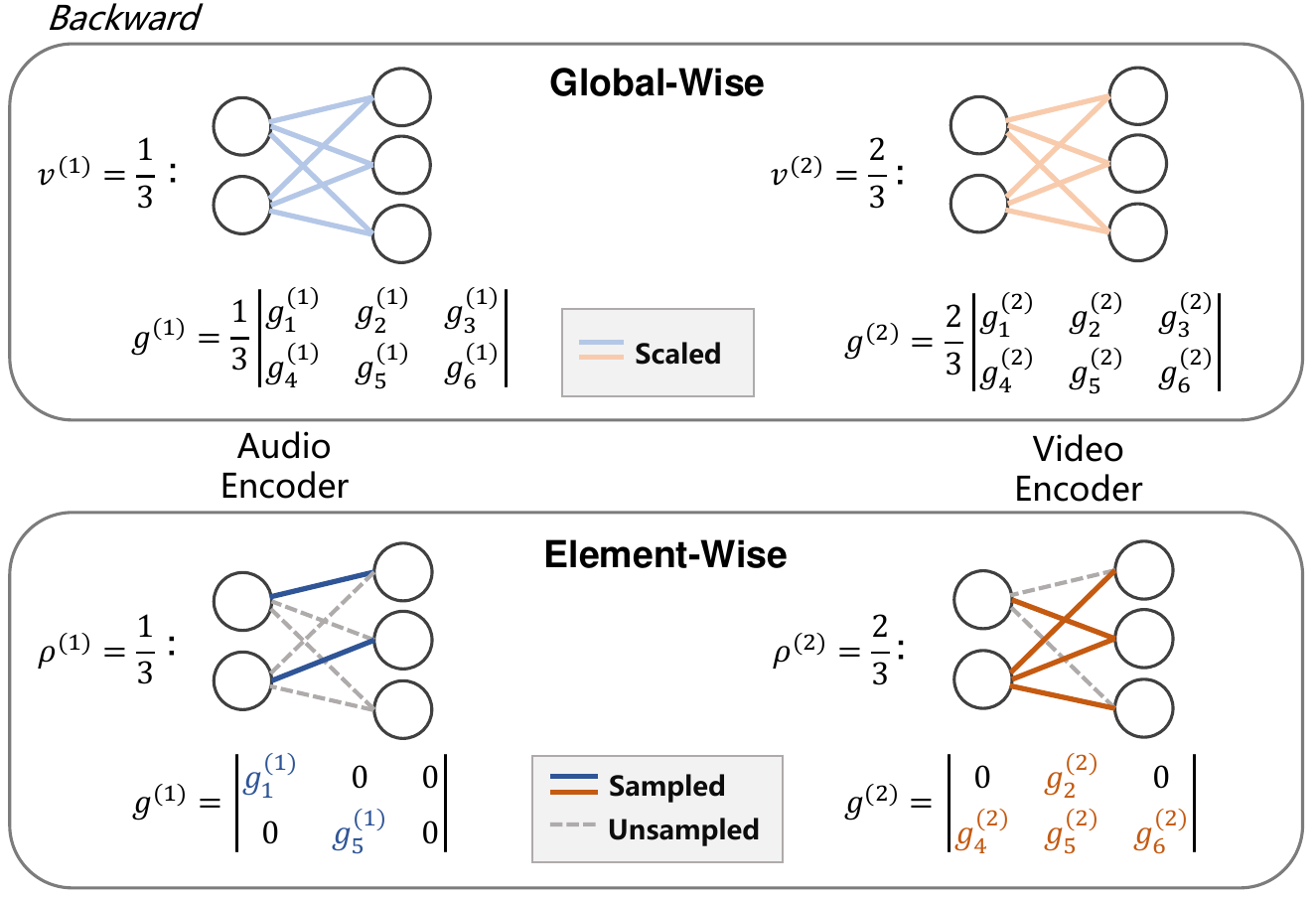} 
\caption{The illustration of different gradient modulation. \textbf{Global-wise}: During backward propagation, uniform modulation is applied to gradients for all parameters. \textbf{Element-wise}: Forward propagation across the entire network, while in backward propagation, parameter gradients undergo differential modulation through a mask subnetwork.}
\label{fig:introduction}
\end{figure}

\begin{table*}
\caption{We evaluate the performance of a multi-modal joint training model under different gradient modulation strategies and intensities for each modality. 
The accuracy of the state-of-the-art method is 67.10. Results superior to the state-of-the-art methods are indicated by \underline{underline}.}
\label{tab:intro}
\begin{subtable}{0.5\textwidth}
\centering
\renewcommand{\arraystretch}{1.5}
\centering

\begin{tabular}{cc|ccccc}
\Xcline{1-7}{0.7pt}
\multicolumn{2}{c|}{\multirow{1}{*}{Global-wise}} & \multicolumn{5}{c}{{ $v^{(2)}$}} \\
\multicolumn{2}{c|}{\multirow{1}{*}{(Accuracy)}}& {1.0} & {0.8} & {0.6} & {0.4} & {0.2} \\
\hline
& 1.0 & \myblue{64.55} & \myora{64.01} & \myora{63.16} & \myora{61.53} & \myora{59.91} \\
&0.8 & \myred{64.74} & \myblue{64.28} & \myora{63.00} & \myora{62.27} & \myora{60.53} \\
 \multirow{1}{*}{$v^{(1)}$}&0.6 & \myred{65.01}& \myred{64.74} & \myblue{63.97} & \myora{62.62} & \myora{60.38} \\
&0.4 & \myred{65.24} & \myred{64.90} & \myred{64.51} & \myblue{63.39} & \myora{61.77} \\
&0.2 & \myred{66.13} & \myred{65.28} & \myred{65.13} & \myred{63.89} & \myblue{63.20} \\
\Xcline{1-7}{0.7pt}
\end{tabular}
\end{subtable}%
\begin{subtable}{0.5\textwidth}
\centering
\renewcommand{\arraystretch}{1.5}
\centering

\begin{tabular}{cc|ccccc}
\Xcline{1-7}{0.7pt}
\multicolumn{2}{c|}{\multirow{1}{*}{Element-wise}} & \multicolumn{5}{c}{{$\rho^{(2)}$}} \\

\multicolumn{2}{c|}{\multirow{1}{*}{(Accuracy)}}& {1.0} & {0.8} & {0.6} & {0.4} & {0.2} \\
\hline
& 1.0 & \myblue{64.55} & \myora{63.74} & \myora{63.32} & \myora{62.93} & \myora{59.72} \\
&0.8 & \myred{64.48} & \myblue{64.44} & \myora{64.40} & \myora{62.62} & \myora{60.80} \\
 \multirow{1}{*}{$\rho^{(1)}$}&0.6 & \myred{65.83} & \myred{65.13} & \myblue{65.02} & \myora{62.62} & \myora{61.65} \\
&0.4 & \underline{\myred{67.14}} & \myred{66.33} & \myred{66.60} & \myblue{63.86} & \myora{62.85} \\
&0.2 & \underline{\myred{68.96}} & \underline{\myred{68.53}} & \underline{\myred{68.84}} & \myred{66.25} & \myblue{63.97} \\
\Xcline{1-7}{0.7pt}
\end{tabular}
\end{subtable}\vspace{-10pt}
\end{table*}

Recent studies~\cite{peng2022balanced, fan2023pmr} reveal that multi-modal approaches maybe perform far from their upper bound even though outperform the single-modal approaches, or even inferior to the single-modal model in certain situations~\cite{wang2020makes}. 
This phenomenon is caused by the notorious ``modality imbalance'' problem~\cite{peng2022balanced, HuangLZYH22} during training, which involves the presence of a dominant modality and a non-dominant modality. 
Therefore, in the multi-modal joint training, due to the inherent greediness~\cite{wu2022characterizing}, the model updates excessively lean towards the dominant modalities, neglecting the learning of the non-dominant modality. Consequently, the non-dominant modality experiences severely slow learning, resulting in the performance of multi-modal learning inferior to that achieved in single-modal learning. 
Similar phenomena have been observed across various multi-modal tasks~\cite{wang2020makes, VielzeufLPJ18, abs-2006-15955}. Therefore, the inefficiency in leveraging and fusing information from diverse modalities poses a significant challenge to the field of multi-modal learning.

{To track this problem, initial work~\cite{wang2020makes} found that different modalities may suffer from overfitting and converge at different rates, which leads to inconsistent learning efficiency when directly optimizing a uniform objective across different modalities. Furthermore, \cite{wu2022characterizing} introduced that the model tends to learn the dominant modality while neglecting the learning of non-dominant modalities, which affects the full utilization of multi-modal information. To cope with this issue, currently, the majority of studies~\cite{wu2022characterizing, peng2022balanced, yao2022modality, fan2023pmr, li2023boosting} have been proposed to modulate each modal gradient during the back-propagation process, either assigning different learning rates to different modal branches or introducing additional losses for each modality. These strategies aim to maximize the contribution of each modality in multi-modal learning. Overall, these methods consistently utilize coarse-grained control at the modal level (global-wise) by updating the complete parameters of each modality. An example of global-wise updating mechanism is presented in the top panel of Figure~\ref{fig:introduction}, where $v^{(1)}/v^{(2)}$ represents the gradient modulation coefficient of the Audio/Video encoder. Within this mechanism, a uniform scaling factor is applied to scale gradients for parameters within the same modality, disregarding the distinctions in importance among different parameters. Inspired by recent advanced progress in subnetwork optimization~\cite{WanZZLF13,LeeCK20}, we explore a fine-grained subnetworks optimization~(element-wise updating) to update the gradients during backward procedure. 
In contrast to the global-wise updating mechanism, which utilizes the gradient modulation coefficient $v$ to scale all gradients, the element-wise updating mechanism updates partial gradients according to a specified parameter update ratio $\rho$.
As illustrated in the bottom panel of Figure~\ref{fig:introduction}, $\rho^{(1)}/\rho^{(2)}$ represents the gradient update ratio of the Audio/Video encoder, and the gradients to be updated are generated through uniform sampling. We compare the global-wise and element-wise updating mechanisms through a preliminary experiment. Specifically, we adopt a multi-modal classification task and take the Kinetics-Sound dataset~\cite{arandjelovic2017look} which includes audio~(dominant) and video~(non-dominant) modalities as an example. The multi-modal network employs concatenation fusion at the last layer of each uni-modal stream before prediction. Table~\ref{tab:intro} illustrates the comparison between random strength gradient global modulation across different modalities and parameter mask element-wise modulation. The results reveal that in the majority of cases, the element-wise modulation strategy outperforms its counterparts. In certain instances, it even surpasses the performance of current state-of-the-art methods.


Drawing inspiration from importance sampling~\cite{zhao2015stochastic}, we intend to refine the direct uniform parameters sampling strategy by dynamically adapting it based on the training data.
Building upon this concept, we turn to optimizing each modal subnetwork, thereby fine-grained stimulating the non-dominant modality and alleviating the suppression from the dominant modality. To this end, we design the Adaptively Mask Subnetworks strategy considering modal Significance~(AMSS). Specifically, we first introduce a simple yet effective mechanism to capture the batch-level significance of each modality, by calculating the mutual information rate with each modal prediction. Different from existing imbalanced multi-modal learning that directly weights the entire parameter gradients of each modality, we mask different sizes of promising parametric subnetworks, through non-uniform adaptive sampling (NAS) for each modality based on the modal significance. Therefore, the non-dominant modality masks a smaller subnetwork, while the dominant modality masks a larger subnetwork. After selection, we then perform partial gradient updates for each modality after masking the gradient of subnetwork parameters. Different from pruning operations, our mask strategy involves differentially updating various parameters within the model during back-propagation. Throughout the model forward process, we still use all parameters to calculate the loss. Furthermore, we present a theoretical validation of the reliability of the AMSS optimization strategy via convergence analysis. Nonetheless, this validation is contingent upon certain assumptions, leading to biased estimation. To mitigate this issue, we introduce an enhanced optimization strategy termed AMSS+, grounded in unbiased estimation principles, thereby circumventing constraints imposed by specific assumptions.
In summary, the main contributions of this paper are summarized as follows:
\begin{itemize}

    \item Based on the preliminary experiment, we propose a novel element-wise updating method called AMSS to solve the modality imbalance problem. AMSS can fine-grained stimulate the non-dominant modality and alleviate the suppression from the dominant modality. To the best of our knowledge, this is the first work that adopts element-wise updating mechanism in multi-modal learning.
    \item We engage in theoretical analysis to showcase the effectiveness of subnetwork update strategies in imbalanced multi-modal learning. Additionally, drawing from theoretical findings, we introduce a novel mask strategy under unbiased estimation, termed AMSS+.
    \item We conduct extensive experiments across various modal scenarios, clearly demonstrating the effectiveness of fine-grained subnetworks optimization optimization in achieving a balanced learning approach for a multi-modal network.
\end{itemize} }
\section{Related Work}
\subsection{Multi-Modal Learning}
Multi-modal learning aims to fuse different modalities from diverse sources~\cite{BaltrusaitisAM19, YangWZL019, YangZWLXJ21}. Existing multi-modal methods commonly employ model-agnostic approaches, which can be classified based on fusion stages~\cite{AtreyHEK10, BaltrusaitisAM19} into early fusion~\cite{perez2018film, NieYSW21, LiangQGCL22, ZengYMH24}, late fusion~\cite{liu2018late, DuWHLC22, AlfaroContrerasVIC23, LiuHXXZD24}, and hybrid fusion~\cite{joze2020mmtm, ZhengTWH023}. In detail, early fusion methods fuse features after extraction through either simple or meticulously designed networks. For example, concatenating features from different modalities to create a joint representation or applying affine transformations~\cite{perez2018film} to features for adaptive influence on the neural network's output has been explored. Additionally, late fusion, also termed prediction fusion, performs integration after each model makes predictions. Many studies have explored sophisticated late fusion methods beyond basic operations like summation and mean operations. For instance, \cite{yang2019comprehensive} introduced an additional attention network for adaptive weight assignment to modality predictions, \cite{NagraniYAJSS21} employed an innovative Transformer-based late fusion architecture with bottlenecks as channels for inter-modal information interaction. Meanwhile, Hybrid fusion~\cite{joze2020mmtm} endeavored to amalgamate the advantages of both approaches within the architecture and enhance the late fusion framework by incorporating a multi-modal transfer module to increase the interaction among features. However, the effectiveness of these approaches relies on the assumption that each modality makes a sufficient contribution throughout the joint training process.

\subsection{Imbalanced Multi-Modal Learning}
Recent studies~\cite{wang2020makes, du2021improving, sun2021learning,peng2022balanced, fan2023pmr, li2023boosting} have demonstrated that despite having access to more information in multi-modal learning, the performance improvement is often limited, and in some cases even worse than single-modal learning. Therefore, imbalanced multi-modal learning aims to rebalance the fitting speeds of non-dominant and dominant modalities, ensuring that each modality is fully utilized during the model training process. Based on this idea, \cite{wang2020makes} proposed a gradient blending technique that assigns different weights to branches based on the overfitting behavior of each modality, which aims to achieve optimal gradient blending. \cite{peng2022balanced, li2023boosting} utilized a dynamic gradient modulation strategy, which reduces the learning rate of the dominant modality. \cite{fan2023pmr} attempted to stimulate non-dominant modalities using prototype cross-entropy and employ prototypical entropy regularization to reduce dominant modality suppression. However, they always directly control the entire parameters update of each modality with unified weight. Furthermore, several attempts tried to employ extra networks to help multi-modal learning. For example, \cite{du2021improving} introduced additional uni-modal branches and distilled the features from these branches into the multi-modal network, \cite{wu2022characterizing} introduced a hierarchical interaction module based on inter-modal information gain for modal representation learning, which can stabilize the differences between modalities during the training process. Nevertheless, the addition of extra modules introduces intricacies into the training procedure.

\subsection{Subnetwork Optimization}
Subnetwork optimization strategies have demonstrated their effectiveness in mitigating the challenge of overfitting in deep neural networks. Current research in this domain can be categorized into two distinct approaches. The first category involves network pruning strategies during the forward process, exemplified by techniques such as DropConnect~\cite{WanZZLF13}, Gaussian Dropout~\cite{BlumHP15}, and Spatial Dropout~\cite{TompsonGJLB15}. These methodologies entailed the pruning of network nodes or connections, thereby reducing the scale of neural networks. The second category of methods~\cite{LeeCK20, xu2021raise, Zhang0LZZJ22} entailed retaining the complete set of network parameters for learning during the forward pass, and then specific neuron gradients are strategically masked to prevent their updates during the back-propagation. Nonetheless, it is worth noting that all subnetwork optimization techniques are devised with a focus on individual training models and are difficult to extend to the joint training of multi-modal models under scenarios characterized by data imbalance.

\section{Rebalance Multi-Modal Network}
In this section, we introduce the mask subnetworks~(MS) for multi-modal learning in subsection~\ref{sec:multimodal}. Then we employ MS for the multi-modal model to balance the learning speed across different modalities. Consequently, we propose the method named Adaptively Mask Subnetworks considering modal significance (AMSS). In subsection~\ref{sec:mehotd_details}, we elaborate on the adaptive construction process of mask subnetworks of the multi-modal model. In subsection~\ref{sec:theory}, we provide the theoretical convergence of AMSS and additionally introduce AMSS+, which is based on unbiased estimation.

\subsection{Preliminary}
\label{sec:multimodal}

For simplicity, we use boldface lowercase letters like ${\bf z}$ to denote the vectors, and the $i$-th element of ${\bf z}$ is denoted as $z_i$. Boldface uppercase letters like $\bf Z$ denote matrices and the element in the $i$-th row and $j$-th column of $\bf Z$ is denoted as $Z_{ij}$. We use numerical superscripts inside parentheses to denote specific modality, e.g., ${\bf z}^{(k)}$ denotes the $k$-th modality of $\bf z$. The notation $\mathcal{S}^{(k)}$ represents the parameter subnetwork set for the $k$-th modality. $\mathbb{E}(\cdot)$, $\mathbb{H}(\cdot)$ and $\mathbb{I}(\cdot)$ denote the exception, information entropy and mutual information, respectively. Furthermore,  we use the symbol $\odot$ to denote the Hadamard product~(i.e., element-wise product) of vectors/matrices.

Without any loss of generality, we first represent the training set as $\mathcal{D} =\left\{\left({\x}_1,{\y}_1\right),\left({\x}_2, {\y}_2\right), \cdots,\left({\x}_{N},{\y}_{N}\right)\right\}$, each example $\x_i$ is with $K$ modalities, i.e., ${\x}_{i}=\{\x_i^{(k)}\}_{k=1}^K$, $\y_i \in \{0,1\}^C$, $C$ is the class number. The goal is to use this dataset $\mathcal{D}$ to train a model that can predict $\y_i$ accurately.

Most multi-modal deep neural networks~\cite{NagraniYAJSS21, fan2023pmr, HanZFZ23} adopt multiple separate branches for the final prediction. These branches consist of multiple feature encoders, $\big\{\varphi^{(k)}(\x_i^{(k)})\big\}_{k=1}^K$, which aim to extract representations from the $\x_i$ data for each modality. Then, the multi-modal fusion operation can be denoted as $\varphi(\x_i)=[\varphi^{(1)}(\x_i^{(1)}); \varphi^{(2)}(\x_i^{(2)});\cdots;\varphi^{(K)}(\x_i^{(K)})]^\top$. Therefore, the final prediction in the multi-modal approach can be expressed as:
\begin{equation}
q(\x_i) = f(\varphi(\x_i)),\nonumber
\end{equation}
where $f$ denotes the classifier. Finally, our objective is to train the $f, \varphi$ to predict $\y$ based on $\x$, by minimizing the loss between the prediction and the ground truths:
\begin{equation}
\mathcal{L} = -\frac{1}{N} \sum_{i=1}^N{\bf y}_{i}^\top \log q(\x_i).\nonumber
\end{equation}

%
Considering the imbalanced multi-modal learning, previous research, such as OGM~\cite{peng2022balanced} and AGM~\cite{li2023boosting}, usually concentrate on modulating model gradients by manipulating learning rates, specifically assigning a lower gradient coefficient to dominant modalities, to rebalance multi-modal learning, which has been verified to be a valid way. 

In the rest of this section, we provide the details of the gradient updating with rebalanced strategy. For simplicity, we adopt the $k$-th modality for illustrating and omit the superscript ``$(k)$''. We represent the parameters at the $t$-th iteration as ${\bf w}({t})$. The parameter for a specific modality is updated by stochastic gradient descent~(SGD):
\begin{equation}
{\bf w}({t+1})={\bf w}({t}) - v(t) \cdot \eta \nabla \mathcal{L}({\bf w}(t)),\nonumber
\end{equation}
where ${\bf w}$ denotes the vectorized parameters of our model, $\nabla \mathcal{L}({\bf w}(t))$ is the gradient of loss function $\mathcal{L}({\bf w})$ at ${\bf w}(t)$ and $\eta>0$ is the learning rate. $v(t)$ is gradient modulation coefficient at the $t$-th iteration. If a modality is dominant, $v(t)<1$. Otherwise, $v(t)=1$ and the equation is equivalent to the standard parameter update method. However, these methods multiply the weight to each model gradient equally (but not all parameters contribute equally to the optimization of training), which we argue is sub-optimal and will lead to excessive computation. Therefore, we put forward MS that determines a gradient update subnetwork ${\mathcal{S}}(t)$ during the $t$-th iteration:
\begin{equation}
m_j{(t)}= \left\{\begin{array}{ll}
1,  & \text{if}\;\;w_j(t) \in {{\mathcal{S}}(t)}, \\
0,  & \text{otherwise},
 \end{array}\right.\nonumber
\end{equation}
where ${\bf m}(t)$ is the 0/1 mask vector with the same size of ${\bf w}$. The parameter updating can be reformulated as follows:
\begin{equation}
{\bf w}(t+1)={\bf w}(t)-\eta \nabla \mathcal{L}({\bf w}(t)) \odot {\bf m}(t).
\label{con:enhance_final}
\end{equation}

\begin{figure*}[htbp]
  \centering
  \includegraphics[trim=0 0 0 0, clip, width=0.9\textwidth]{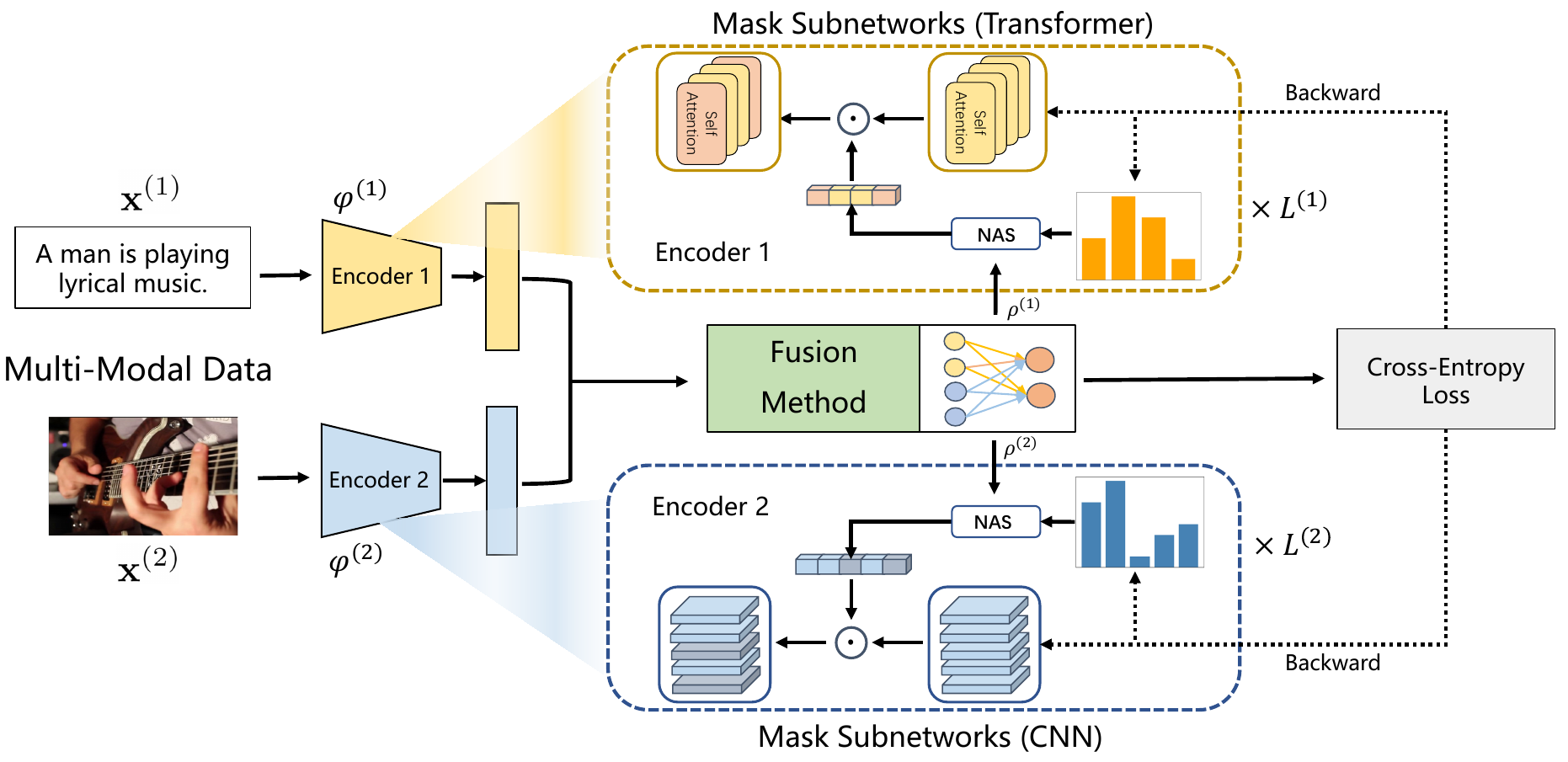}
  \vspace{-5pt}
  
  \caption{Overall framework of our proposed AMSS strategy, using the Transformer-CNN structure as an example.
}
  \vspace{-15pt}
  \label{fig:model}
\end{figure*}


\subsection{AMSS}
\label{sec:mehotd_details}
Based on preliminary experiments, we found that the MS strategy with uniform sampling performs excellently as shown in Table~\ref{tab:intro}. Additionally, inspired by importance sampling, which involves selecting parameters based on input data, we propose Adaptively Mask Subnetworks considering Modal Significance (AMSS), as illustrated in Figure~\ref{fig:model}.
The most crucial challenge of AMSS is how to adaptively obtain $\mathcal{S}^{(k)}(t)$ for each modality. This process primarily encompasses two problems: 1) how many parameters need to be selected for each modality and 2) what are the criteria for selecting parameters. Further details will be elaborated in the following content.

\subsubsection{Parameter Quantity Mask via Modal Significance}
In most multi-modal learning tasks, both modalities are assumed to be predictive of the target. Therefore, information theory can naturally state the modal significance, in particular, the mutual information~\cite{SridharanK08} $\mathbb{I}({\bf X}^{(k)}; {\bf Y})$ (between each modality and ground-truth) measure how much information is shared between ${\bf X}^{(k)}$ and ${\bf Y}$, which can be viewed as how much knowing ${\bf Y}$ reduces our uncertainty of ${\bf X}^{(k)}$. 

Furthermore, considering the strong correlation between mutual information and the information content in each modality, we turn to employ the straightforward yet effective mutual information rate for evaluation, which has placed greater emphasis on the inherent changes within each modality during the training process. The equation can be represented as:
\begin{equation}
\hat{u}^{(k)}=\frac{\mathbb{I}({\bf X}^{(k)}; {\bf Y})}{\mathbb{H}({\bf X}^{(k)})},\nonumber
\end{equation}
where ${\bf X}^{(k)}$ denotes a batch of $k$-th modal examples, ${\bf Y}$ denotes the set of corresponding ground-truths. $\mathbb{H}({\bf X}^{(k)})$ denotes the information entropy of of $k$-th modality. Intuitively, when the mutual information rate of a modality is higher, it signifies that this modality has a significant impact on the predictive task. This is because this modality can reduce uncertainty in the predictive task, thereby enhancing predictive capability. However, due to the inherent challenge of directly estimating mutual information in high-dimensional space as~\cite{BelghaziBROBHC18}, conducting such an estimation is practically infeasible. To facilitate the practical utilization of mutual information, we employ variational bounds to approximate its true value following~\cite{ChengHDLGC20}:

\begin{equation}
\label{con:mutual}
\mathbb{I}({\bf X}^{(k)}; {\bf Y}) \geq \mathbb{H}({\bf Y}) + \mathbb{E}_{p(\x^{(k)}, \y)}[\log q(\y \mid \x^{(k)})], \\ 
\end{equation}
\begin{equation}
\begin{split}
\mathbb{E}_{p(\x^{(k)}, \y)}[\log q(\y \mid \x^{(k)})] &= 
\frac{1}{B} \sum_{i=1}^B \log q\left(\y_i\mid \x_i^{(k)}\right), \\
\mathbb{H}({\bf Y}) &=  -\sum_{c=1}^C p(\y_c) \log p(\y_c),
\end{split}\nonumber
\end{equation}
where, $B$ denotes the batch size, $\mathbb{H}({\bf Y})$ represents the information entropy of the ground-truth ${\bf Y}$ and $p(\y_c)= \frac{1}{B} \sum_{i=1}^B y_{i_c}$. The Barber-Agakov lower bound in Equation \ref{con:mutual} or is tight, i.e., there is no gap between the bound and truth value, when $p(\y_i\mid \x_i^{(k)})= q(\y_i\mid \x_i^{(k)})$. Therefore, it is necessary to train a classifier predicting $q(\y_i \mid \x_i^{(k)})$ to approximate $p(\y_i\mid \x_i^{(k)})$. 
For different fusion methods, the computation of $q(\y_i\mid \x_i^{(k)})$ varies. 
In late fusion, $q(\y_i \mid \x_i^{(k)}) = \mathrm{softmax}(f^{(k)}(\varphi^{(k)}(\x_i^{(k)}))_{\hat{y}_i}$, where $\hat{y}_i= \mathrm{argmax}(\y_i)$ and each modality has its own classifier $f^{(k)}$ for prediction. In early and hybrid fusion, we use zero-padding to represent features excluding the \(k\)-th modality, i.e., \(\{\varphi^{(n)}(\x_i^{(n)}) = \textbf{0}^{(n)}\}_{n\in \mathcal{N}/\{k\}}\), where \(\mathcal{N} = \{1,2,\ldots,K\}\) is the set of all modalities. Subsequently, \(q(\y_i \mid \x_i^{(k)}) = \mathrm{softmax}(f(\varphi(\x_i)))_{\hat{y}_i}\).

Besides, to reduce fluctuations and stabilize the learning process, we employ the momentum update method, which evaluates the significance of the current modalities during the training phase by accumulating the modal significance obtained from the historical training data:
\begin{equation}
u^{(k)}= \lambda u^{(k)} + (1-\lambda) \hat{u}^{(k)},
\label{con:contribution_u}
\end{equation}
where $\lambda$ is the attenuation factor. To address the problem of modality imbalance, we need to provide a greater advantage to the non-dominant modality during network updates while suppressing the dominant modality. Therefore, we mask fewer parameters for the non-dominant modality and comparatively more parameters for the dominant modality. Based on this idea, we first design the update ratio based on the modal significance $u^{(k)}$:
\begin{equation}
\begin{split}
\rho^{(k)}=1 - \frac{\exp (u^{(k)}/\tau) }{\sum_{n=1}^{K} \exp (u^{(n)}/\tau)}, \\
\label{con:mask_ratio}
\end{split}
\end{equation}

where $\tau$ serves as a hyper-parameter designed to adjust the size disparities across subnetworks of different modalities. Specifically, setting $\tau < 1$ amplifies these disparities, whereas $\tau > 1$ diminishes them. Furthermore, insights drawn from Table 1 suggest a preliminary conclusion: a higher ratio of non-dominant modal parameter updates relative to dominant ones often leads to improved model performance. Consequently, it is advisable to select $\tau < 1$ in experimental setups to amplify the disparities among subnetworks, a strategy that our subsequent experiments have confirmed to be effective.

\subsubsection{Task-Guided Parameter Mask Criteria}
As demonstrated in~\cite{SouratiAM19, SinghA20, xu2021raise}, it is evident that parameters with Fisher information estimation~\cite{fisher1922mathematical} play a pivotal role in learning the target task. Therefore, we adopt the Fisher information estimation as the selection criteria, which can provide an estimation of how much information a random variable carries about a parameter of the distribution~\cite{TuBCS16}, and measure the relative
significance of the parameters. In particular, the Fisher information of ${\bf w}^{(k)}$ can be represented as:
\begin{equation}\Scale[0.985]{
\mathbf{F}({\bf w}^{(k)})=\mathbb{E}\big[\big(\frac{\partial \log p(\hat{y} \mid \x^{(k)} ; {\bf w}^{(k)})}{\partial {\bf w}^{(k)}}\big)\big(\frac{\partial \log p(\hat{y} \mid  \x^{(k)} ; {\bf w}^{(k)})}{\partial {\bf w}^{(k)}}\big)^{\top}\big].}\nonumber
\end{equation}

Actually, $\mathbf{F}({\bf w}^{(k)})$ can be viewed as the covariance of the gradient of the log-likelihood with respect to the parameters ${\bf w}^{(k)}$. Following~\cite{kirkpatrick2017overcoming}, given the batch data, we use the diagonal elements of the empirical Fisher information matrix to estimate the significance of the parameters. Formally, we derive the Fisher information for the $j$-th parameter as follows:
\begin{equation}
{F}_{j}({\bf w}^{(k)})=\frac{1}{B} \sum_{i=1}^{B}\big(\frac{\partial \log p(\hat{y}_i \mid \x^{(k)}_i ; {\bf w}^{(k)})}{\partial w^{(k)}_j}\big)^2.\nonumber
\end{equation}

Furthermore, we normalize the diagonal element as the importance of parameters:
\begin{equation}
p^{(k)}_j=\frac{{F}_{j}({\bf w}^{(k)})}{\sum_{j=1}^{|{\bf w}^{(k)}|} {F}_{j}\left({\bf w}^{(k)}\right)}.\nonumber
\end{equation} 

As a result, the more important the parameter towards the target task, the higher $p^{(k)}_j$ it conveys. We set a probability distribution ${\bf P}^{(k)}=\{p^{(k)}_1, p^{(k)}_2, \ldots, p^{(k)}_{|{\bf w}^{(k)}|}\}$. Then, we employ the non-uniform adaptive sampling~\cite{gopal2016adaptive} without replacement for parameter selection. This approach allows $\mathcal{S}^{(k)}$ to concentrate on parameters associated with high information content. In contrast to the method of directly selecting the highest information parameters, it can encompass a more extensive scope of parameters through probabilistic sampling. Subsequently, based on the update ratio $\rho^{(k)}$, we conduct the sampling process to construct the parameter subnetwork set for each modality:
\begin{equation}
\mathcal{S}^{(k)}=\{s^{(k)}_1, s^{(k)}_2, \ldots, s^{(k)}_{\left\lceil \rho^{(k)}*|\w^{(k)}|\right\rceil}\},\nonumber
\end{equation}
where $s^{(k)}_i$ represents a parameter in the $k$-th modal network that is selected during the $i$-th sampling.
Subsequently, combining with Equation \ref{con:enhance_final}, we can derive parameter the gradient update strategy of AMSS as follows:
\begin{equation}
\begin{split}
{\bf w}^{(k)}(t+1)={\bf w}^{(k)}(t)-\eta \nabla \mathcal{L}({\bf w}^{(k)}(t)) \odot {\bf m}^{(k)}(t).\nonumber
\end{split}
\end{equation}

Given the considerable quantity of parameters in encoders, executing non-uniform adaptive sampling at the network parameter level is practically unfeasible due to its significant time overhead. To address this, we introduce the concept of the mask unit as a replacement for masking individual parameters and perform sampling based on this. Here we use the ResNet network (CNN-based network)~\cite{he2016deep} as an example. In Resnet architecture , each convolutional layer deploys multiple convolutional kernels to extract features spanning different scales, in which we define a mask unit corresponding to a convolutional kernel. We consider the parameters within the mask unit as a set, simultaneously masking or selecting the parameters within it. The previous notion of parameter importance is transformed into convolutional kernel importance, which can be computed by summing the Fisher information of all parameters within a convolutional kernel. Furthermore, in order to prevent the situation where a majority of convolutional kernels in a layer remain unsampled, we apply the update ratio to each convolutional layer, performing sampling across them. Note that the Transformer-based network can be masked similarly to the CNN-based network, by considering the masked self-attention heads in the multi-head attention module.  

\subsection{Theoretical Analysis and Improved AMSS}
\label{sec:theory}
In this subsection, we conduct a theoretical analysis of the convergence properties for updating parameters in the Mask Subnetwork under the non-convex optimization setting. Based on Equation~\ref{con:enhance_final}, we suppose that the stochastic gradient $\nabla \ell({\bf w}(t))$ is unbiased, i.e., $\mathbb{E}[\nabla \ell({\bf w}(t))] = \nabla\mathcal{L}({\bf w}(t))$, which is commonly used in non-convex optimization. However, under the 0/1 mask strategy, $\nabla \ell({\bf w}(t))$ and ${\bf m}(t)$ are not independent, the stochastic gradient $\nabla \ell({\bf w}(t)) \odot {\bf m}(t)$ is biased, that is, $\mathbb{E}[\nabla \ell({\bf w}(t))  \odot {\bf m}(t)] \neq \nabla\mathcal{L}({\bf w}(t))$. Under the Mask-Incurred Error assumption, we have the following convergence result for AMSS. We include more details and the proof in Appendix.
\begin{thm}[Informal, AMSS]
Under some assumptions for the stochastic gradient $\nabla \ell({\bf w}(t)) \odot {\bf m}(t)$, we have
\begin{align}
\frac{1}{T}\sum_{t=1}^{T}\|\nabla\mathcal{L}({\bf w}(t)) \|^2 \le O\left(\frac{1+(1+\nu)^2}{\sqrt{T}(1+\nu)(1-\delta^2)}\right), \nonumber
\end{align} 
\noindent where $\delta \in (0,1)$ and $\nu\ge0$ are two constants.
\end{thm}
The result shows that AMSS converges to a stationary point with the rate of $O\left(\frac{1+(1+\nu)^2}{\sqrt{T}(1+\nu)(1-\delta^2)}\right)$, which is less worse than the $O\left(\frac{1+(1+\nu)^2}{\sqrt{T}(1+\nu)}\right)$ (i.e., $\delta=0$, without masking) due to the Mask-Incurred Error assumption. To relax this assumption, we tend to propose another importance sampling strategy. 

To this end, we propose a novel mask strategy that eliminates the bias of stochastic gradient $\nabla \ell({\bf w}(t)) \odot {\bf m}(t)$. 
\begin{equation}
\hat{m}_j({t})= \left\{\begin{array}{ll}
\frac{1}{p_j},  & {\text{if}}\;\;w^{(k)}_j(t) \in {{\mathcal{S}^{(k)}}}(t), \\
0,  & \text{otherwise}.
\end{array}\right.
\label{con:Mask_enhance}
\end{equation}
Similarly, we have the following convergence result for the new smapling method. 
\begin{thm}[Informal, AMSS+]
Under some assumptions for the stochastic gradient $\nabla \ell({\bf w}(t)) \odot \hat{\bf m}(t)$, we have
\begin{align}
\frac{1}{T}\sum_{t=1}^{T}\|\nabla\mathcal{L}({\bf w}(t)) \|^2 \le O\left(\frac{1+(1+\nu)^2}{\sqrt{T}(1+\nu)}\right), \nonumber
\end{align}
\noindent where $\delta \in (0,1)$ and $\nu\ge0$ are two constants.
\end{thm}
Inspired by the theoretical foundation, we propose an improved method AMSS+. However, when the dimension of $\mathbf{m}$ is high, the value of $p_j$ is usually too small so that $\frac{1}{p_j}$ is too large, which may lead to gradient explosion. Thus, in practice, we replace $\frac{1}{p_j}$
by $\frac{1}{p_j+a}$, where $a\in [0,1)$ is a constant used to eliminate gradient explosion. 
To facilitate the selection of the hyper-parameter $a$, we directly substitute the number of mask units $L_l^{(k)}$ of $l$-th layer in the network for $a$. Therefore, the Equation~\ref{con:Mask_enhance} can be modified to:
\begin{equation}
\hat{m}^{(k)}_j(t)= \left\{\begin{array}{ll}
\frac{1}{p^{(k)}_j+L_l^{(k)}},  & {\text{if}}\;\;w^{(k)}_j(t) \in {{\mathcal{S}^{(k)}}}(t), \\
0,  & \text{otherwise}. \nonumber
\end{array}\right.
\end{equation}
Considering mask subnetworks under the unbiased estimation theory, we further propose the parameter gradient update strategy of AMSS+ as follows:
\begin{equation}
\begin{split}
{\bf w}^{(k)}(t+1)={\bf w}^{(k)}(t)-\eta \nabla \mathcal{L}({\bf w}^{(k)}(t)) \odot \hat{\bf m}^{(k)}(t).\nonumber
\end{split}
\end{equation}

In summary, the aforementioned paragraph elucidates the strategies of both AMSS and AMSS+, with the latter distinguished by its utilization of the 0/$\frac{1}{p_j}$ masking approach.
%

\section{Experiment}
\begin{table*}[!htbp]
    \centering
    \caption{Comparison between AMSS with other SOTA methods on four datasets. The optimal performances are highlighted in \textbf{bold}. The  \underline{underscore} symbol represents the second best performance.}  
    \label{tab:Sota_compare}
    \normalsize
    \begin{tabular}{cl|cc|cc|cc|cc}
    \Xcline{1-10}{0.7pt}
    \multicolumn{2}{c|}{\multirow{2}{*}{Methods}} &
          \multicolumn{2}{c|}{{Kinetics-Sound}} &
          \multicolumn{2}{c|}{{CREMA-D}} &
          \multicolumn{2}{c|}{{Sarcasm-Detection}} &
          \multicolumn{2}{c}{{Twitter-15}} \\
     
        & & ACC & mAP & ACC & mAP & ACC & Mac-F1 & ACC & Mac-F1 \\ \midrule
        \multirow{2}{*}{Uni-modal}&Audio/Text & 54.12 & 56.69 & 63.17 & 68.61  & 81.36  & 80.65  & 73.67  & 68.49  \\ 
        &Video/Image & 55.62  &58.37  & 45.83  & 58.79  & 71.81  & 70.73  & 58.63  & 43.33  \\ \hline
        &Concat & 64.55  & 71.30  & 63.31  & 68.41  & 82.86  & 82.40  & 70.11  & 63.86  \\ 
        &Affine & 64.24  & 69.31  & 66.26  & 71.93  & 82.40  & 81.88  & 72.03  & 59.92  \\ 
        &Channel & 63.51  & 68.66  & 66.13  & 71.70  & - & - & - & - \\ 
        {Multi-modal Fusion}  &ML-LSTM & 63.94 & 69.02 & 62.90 & 64.73& 82.77 &	82.05 &	70.68 &	65.64 \\ 
        &Sum & 64.90  & 71.03  & 63.44  & 69.08  & 82.94  & 82.47  & 73.00  & 66.61  \\ 
        &Weight & 65.33  & 71.10  & 66.53  & 73.26  & 82.65  & 82.19  & 72.42  & 65.16  \\ 
        &ETMC & 65.67 &	72.50 &	65.86 &	71.34 &	83.69 &	83.23 &	73.96 &	67.39 
 \\ \hline
        &MSES&64.71 &	70.63 &	61.56 &	66.83 &	84.18 &	83.60 &71.84 &	66.55\\
        &OGR-GB & {67.10}  & 71.39  & 64.65  & 68.54  & 83.35  & 82.71  & 74.35  & 68.69  \\
        &OGM-GE & 66.06  & 71.44  & 66.94  & 71.73  & 83.23  & 82.66  & {74.92}  & 68.74  \\ 
        &Greedy & 66.52  & {72.81}  & 66.64  & 72.64  & - & - & - & - \\
        Multi-modal&DOMFN & 66.25  & 72.44  & {67.34}  & {73.72}  & 83.56  & 82.62  & 74.45  & 68.57  \\ 
        Rebalance Fusion&MSLR&65.91 &	71.96 &	65.46 &	71.38 &	\underline{84.23} &	\underline{83.69} &72.52 	& 64.39 \\
        &PMR & 66.56  & 71.93  & 66.59  & 70.30  & 83.60  & 82.49  & 74.25  & 68.60  \\ 
        &AGM & 66.02 &	72.52 &	67.07 &	73.58 &	84.02 &	83.44 &	74.83 	&{69.11} 
 \\ 
        &AMSS  & \underline{68.96}  & \underline{74.89}  & \underline{67.61}  & \underline{73.97}  & 84.14  & \underline{83.69}  & \textbf{75.89}  & \textbf{69.81}  \\
        &AMSS+ & \textbf{72.25}  & \textbf{79.13}  & \textbf{70.30}  & \textbf{76.14}  & \textbf{84.35}  & \textbf{83.77}  & \underline{75.12}  & \underline{69.23} \\ 
        \Xcline{1-10}{0.7pt}
    \end{tabular}
\end{table*}

\subsection{Experimental Setups}
\textbf{Datasets.} Following the prior research considering multi-modal imbalance~\cite{peng2022balanced, fan2023pmr}, we adopt the \textbf{Kinetics-Sound}~\cite{arandjelovic2017look} and \textbf{CREAM-D}~\cite{cao2014crema} datasets for validation, which includes audio and video modalities. To further validate the effectiveness of the proposed method, our research is extended in two dimensions. Firstly, the analysis is expanded to encompass the text-image modality, incorporating the \textbf{Sarcasm Detection}~\cite{cai2019multi} and \textbf{Twitter-15}~\cite{yu2019adapting} datasets. Secondly, we employ the \textbf{NVGesture}~\cite{MolchanovYGKTK16} dataset to conduct research that goes beyond the limitation of two modalities.

Kinetics-Sound is used for video action recognition. It comprises 31 human action categories. The dataset contains a total of 19,000 10-second video clips (15k training set, 1.9k validation set, 1.9k test set). CREMA-D is designed for speech emotion recognition. It consists of 7,442 original clips. These clips are divided into 6,698 samples for the training set and 744 samples for the test set. CREMA-D can be categorized into six emotions. Sarcasm-Detection is designed for the task of sarcasm detection. It consists of 24,635 text-image pairs (19,816 training set, 2,410 validation set, 2,409 test set). The dataset is categorized into two classes. Twitter-15 is used for emotion recognition tasks, consisting of three classes. The data is collected from Twitter data~\cite{0001FLH18}, which comprises 5338 text-image pairs (3,179 training set, 1,122 validation sets, and 1,037 test set). NVGesture is collected using multiple sensors to investigate human-computer interfaces. It encompasses 1532 dynamic hand gestures (1,050 training set and 482 test set). We randomly sample 20\% of training examples as the validation set, following~\cite{wu2022characterizing}, and use RGB, depth, and optical flow modalities in our setting. This dataset comprises 25 classes of hand gestures.  

\noindent\textbf{Baselines.} We compare two categories of methods: 1) fusion methods considering modal rebalancing strategies, and 2) traditional fusion methods. In detail, modal rebalancing fusion methods include ORG-GB~\cite{wang2020makes}, MSES~\cite{fujimori2020modality}, OGM-GE~\cite{peng2022balanced}, Greedy~\cite{wu2022characterizing}, DOMFN~\cite{yang2022domfn}, MSLR~\cite{yao2022modality}, PMR~\cite{fan2023pmr}, AGM~\cite{li2023boosting}. The traditional fusion methods encompass feature concatenation fusion, affine transformation fusion~\cite{perez2018film}, channel-wise fusion~\cite{joze2020mmtm}, multi-layers lstm fusion~\cite{NieYSW21}, prediction summation fusion, prediction weight fusion~\cite{yang2019comprehensive} and ETMC~\cite{HanZFZ23}.  For convenience in the description, we abbreviate these methods as Concat, Affine, Channel, ML-LSTM, Sum, and Weight. Concat involves concatenating multiple modal features to obtain high-dimensional features, which are then used for downstream tasks. Affine applies a feature-wise affine transformation to intermediate neural network features. Channel recalibrates channel features of different CNN streams through the squeeze and multi-modal excitation steps. ML-LSTM method employs the LSTM structure to fully integrate various modalities. Sum predicts by calculating the probability mean of various modal prediction results. Weight assigns a weight to each individual modal branch using an attention mechanism. ETMC dynamically evaluates the trustworthiness of each modality across various samples, ensuring dependable integration.

\noindent\textbf{Evaluation Metrics.} We use accuracy (Acc) and mean Average Precision (mAP) for audio-video datasets following~\cite{peng2022balanced}. For the text-image dataset and NVGesture dataset, we utilize the Acc and Macro F1-score (Mac-F1) following~\cite{cai2019multi, yu2019adapting}. The Acc measures the proportion of concordance between predicted outcomes and true labels. The Mac-F1 computes the average of F1 scores for each category, while the mAP calculates the mean of average precision for each category. 

\noindent\textbf{Implementation Details.} 
In all our experiments, we use raw data as input. Following~\cite{peng2022balanced,fan2023pmr}, for the Kinetics-Sound and CREMA-D datasets, we use ResNet18~\cite{he2016deep} as the backbone for both modalities. In detail, for the video modality, we extract 10 frames from video clips and uniformly sample 3 frames as the input. The input channels are changed from 3 to 1, as demonstrated in~\cite{chen2020vggsound}. For the audio modality, we convert the data into spectrograms with a size of 257$\times$1004 for Kinetics-Sound and 257x299 for CREMA-D, using the librosa~\cite{mcfee2015librosa}. For the backbone of the text-image datasets, we employ ResNet50 and BERT~\cite{kenton2019bert} for the image and text modality, respectively. We crop the image data to a size of 224$\times$224 and set the maximum sequence length for text data to 128. We use stochastic gradient descent (SGD) as the optimizer for the audio-video and NVGesture datasets, with a momentum of 0.9 and weight decay of 1e-4. The initial learning rate is set to 1e-2, and when the loss is saturated, it is multiplied by 0.1. For the text-image dataset, following~\cite{cai2019multi, yu2019adapting}, we use Adam as the optimizer, with an initial learning rate of 1e-5. In the context of the audio-video datasets, the scaling factor $\tau$ is configured to be 0.25, while on the text-image and NVGesture datasets, the $\tau$ is set to 0.5. Besides, For the NVGesture dataset, we follow the data preparation steps outlined in~\cite{wu2022characterizing} and employ the I3D~\cite{CarreiraZ17} as uni-modal branches. We train all models on a single RTX 4090 GPU. 
The selection scope of subnetworks $\mathcal{S}^{(k)}$ varies depending on different fusion methods, as the parameters corresponding to each modality, i.e. ${\bf w}^{(k)}$, vary under different fusion methods. Therefore, below, we provide details of the parameter composition for each modality \({\bf w}^{(k)}\), under different fusion strategies. In the framework of late fusion, each modality subnetwork encompasses both its classier and encoder, while in other fusion scenarios, each modality subnetwork only considers its encoder.
\begin{table}[h]
    \centering
    \caption{The performance on NVGesture dataset. The involved modalities are RGB, OF, and Depth. Baseline means prediction sum fusion with no gradient modulation strategy. The best results are highlighted in \textbf{bold} and the \underline{underscore} symbol represents the second best performance.} 
    \label{tab:three-modal}
    \normalsize
    \begin{tabular}{l|cc|cc}
    \Xcline{1-5}{0.7pt}
        \multicolumn{1}{l|}{\multirow{3}{*}{Methods}} &
        \multicolumn{2}{c|}{{NVGesture}}&
          \multicolumn{2}{c}{{NVGesture}} \\ 
          ~&\multicolumn{2}{c|}{{scratch}} &
            \multicolumn{2}{c}{{pretrain}} \\
        & ACC  & Mac-F1 & ACC  & Mac-F1  \\ \hline
        RGB & 68.88 & 69.05  & 78.22  & 78.33   \\ 
        OF & 64.11 & 64.34  & 78.63    & 78.65    \\ 
        Depth & 80.50 &	80.41 & 81.54    & 81.83  \\ \hline
        Baseline & 78.63 &	78.91 &  82.57    & 82.68    \\
        MSES & 79.46 & 79.48 & 81.12 & 81.47 \\
        ORG-GB & {81.95} & {82.07}  &  \underline{82.99}   & {83.05}  \\
        MSLR& 81.54& 81.38  & 82.37 & 82.39  \\
        AGM& 80.71 & 81.18   & 82.78  & 82.84  \\ 
        AMSS & \underline{82.57} & \underline{82.65} & \underline{82.99} & \underline{83.08}  \\ 
        AMSS+ & \textbf{84.64} & \textbf{84.94} &  \textbf{83.20}   & \textbf{83.25} \\ \Xcline{1-5}{0.7pt}
    \end{tabular} 
    \vspace{-5pt}
\end{table}

\begin{table*}[!htbp]
    \centering
    \caption{The backbone of the network is transformer-based (MBT). Comparing with other imbalanced multi-modal learning methods. The best results are highlighted in \textbf{bold}. $\downarrow$ indicates a performance decrease compared to the baseline of the MBT model.}
    \vspace{-5pt}
    \label{tab:backbone}
    \normalsize
    \begin{tabular}{l|cc|cc|cc|cc}
    \Xcline{1-9}{0.7pt}
        {\multirow{3}{*}{Methods}} &
          \multicolumn{2}{c|}{{Kinetics-Sound}} &
          \multicolumn{2}{c|}{{CREMA-D}} &
          \multicolumn{2}{c|}{{Kinetics-Sound}} &
          \multicolumn{2}{c}{{CREMA-D}} \\
        ~ & \multicolumn{2}{c|}{{scratch}} &
            \multicolumn{2}{c|}{{scratch}} &
            \multicolumn{2}{c|}{{pretrain}} &
            \multicolumn{2}{c}{{pretrain}} \\
        ~ & ACC & mAP & ACC & mAP & ACC & mAP & ACC & mAP \\ \hline
        MBT & 58.52  & 62.32  & 54.17  & 55.26 & 79.03  & 85.71  & 78.63  & 87.44  \\
        MSES & 59.22 &	63.58 &	54.44 &	58.47 & 79.67 &	86.50 	&78.36 ($\downarrow$) & 87.41 ($\downarrow$) \\
        OGR-GB & 59.18  & 62.27 ($\downarrow$)  & 54.70  & 57.16  & 79.67  & 85.00 ($\downarrow$)  & 79.03  & 87.74   \\ 
        OGM-GE & 57.67 ($\downarrow$) &62.24 ($\downarrow$)  & 54.03 ($\downarrow$) & 54.94 ($\downarrow$) & 78.59 ($\downarrow$)  & 85.78  & 78.23 ($\downarrow$) & 87.63  \\ 
        DOMFN & 58.68  & 62.82  & 53.63 ($\downarrow$)  & 54.70 ($\downarrow$)  &79.40  & 85.81  & 79.44  & 87.49  \\ 
        MSLR& 59.72 &	63.71 &	54.30 &	58.63 &	79.47 &	86.53 & 78.76 &	87.95 \\
        PMR & 58.06 ($\downarrow$)  & 61.71 ($\downarrow$)  & 53.36 ($\downarrow$) & 55.78   &78.78 ($\downarrow$)  & 85.23 ($\downarrow$)  & 78.76 & 87.41 ($\downarrow$)  \\     
        AGM & \textbf{60.34}  & 63.61  & 54.84  & 55.15 ($\downarrow$)&79.36  & 85.56 ($\downarrow$) &79.30 & 87.97  \\
        AMSS & 59.37  & 63.53  & 55.38  & 57.08& 79.51  & 85.94   & 79.44  & 87.95
 \\ 
        AMSS+ &  60.03 & \textbf{64.28}  & \textbf{56.18}  & \textbf{59.51} &  \textbf{80.09}  & \textbf{86.25} & \textbf{79.57} & \textbf{88.10}  
\\ \Xcline{1-9}{0.7pt}
        \end{tabular}
\end{table*}

\vspace{-6pt}
\subsection{Comparison with  Multi-Momal Learning Methods}
To substantiate the advantages of AMSS and AMSS+, we conduct a comprehensive comparison, considering both modal rebalancing methods and traditional fusion approaches. Moreover, we identify limitations in the model architecture of Channel and Greedy constraints, rendering them unsuitable for the transformer structure. Additionally, to ensure experimental fairness, we evaluate AMSS, AMSS+, and all comparative methods using the same backbone, and use Concat for all multi-modal rebalance fusion methods to ensure a unified fusion strategy. We employ experiments on NVGesture from scratch (scratch) and with pre-training (pretrain).
Taking into account the limitations posed by specific comparative methods in scenarios of two modalities, we undertake comparisons on the NVGesture dataset, encompassing MESE, ORG-GB, MSLR, and AGM methods.

The results for both audio-video datasets and text-image datasets are presented in Table~\ref{tab:Sota_compare}, while results of the NVGesture dataset are depicted in Table~\ref{tab:three-modal}. From the results, we draw the following observations: (1) On both the Twitter-15 and NVGesture datasets, we observe the phenomenon where the best uni-modal performance surpasses that of multi-modal joint learning. Besides, on other datasets, fusion methods without rebalancing exhibit limited improvement compared to the best uni-modal performance, especially on the CREMA-D and Sarcasm-Detection datasets. This limitation stems from the challenge of modality imbalance. (2) All modal rebalancing methods exhibit substantial enhancements compared to the traditional feature concatenation fusion. This observation not only highlights the influence of the imbalance phenomenon on performance but also substantiates the effectiveness of the modal rebalance strategy. (3) It is evident that AMSS/AMSS+ consistently achieves superior performance across all metrics compared to other comparison methods. We observe a significant improvement in the performance of AMSS+ on Kinetics-Sound/CREMA-D. After modulating, our method achieves a performance improvement of 5.15\%/2.96\% and 7.70\%/6.99\% in the Accuracy metric compared to the second-best approach and Concat. (4) Differing from modal rebalancing methods restricted to scenarios with only two modalities, such as OGM-GE and Greedy, our approach can address challenges in scenarios involving more than two modalities. In the evaluation of the NVGesture dataset, AMSS+ consistently achieves the best performance compared to other methods designed for multiple modalities. It is worth noting that, unlike other methods, the effectiveness of the AMSS+ in training from scratch is even better than pre-training. This observation confirms the robustness and effectiveness of our proposed method. (5) Compared to the biased estimation approach of AMSS, AMSS+ with unbiased estimation demonstrates superior performance in most scenarios, especially on audio-video datasets. This improvement indicates that AMSS+ leads to more accurate parameter estimation and superior performance. The consistency between experimental results and theoretical predictions bolsters the credibility of our unbiased estimation strategy.

\begin{table*}[!htbp]
    \centering
    \caption{Various fusion methods combined with AMSS. $\dagger$~and~$\ddagger$~indicates that AMSS abd AMSS+ has been applied, respectively. } 
    \vspace{-5pt}
    \label{tab:various_method}
    \normalsize
    
    \begin{tabular}{l|cc|cc|cc|cc}
    \Xcline{1-9}{0.7pt}
        \multicolumn{1}{c|}{\multirow{2}{*}{Methods}} &
          \multicolumn{2}{c|}{{Kinetics-Sound}} &
          \multicolumn{2}{c|}{{CREMA-D}} &
          \multicolumn{2}{c|}{{Sarcasm-Detection}} &
          \multicolumn{2}{c}{{Twitter-15}} \\ 
          & ACC & mAP & ACC & mAP & ACC & Mac-F1 & ACC & Mac-F1  \\ \hline
        Affine & 64.24  & 69.31  & 66.26  & 71.93  & 82.40  & 81.88  & 72.03  & 59.92  \\ 
        Affine$\dagger$ & 65.02  & 71.60  & 66.94  & 71.38  & 83.31  & 82.70  & 72.81  & 65.42  \\ 
        Affine$\ddagger$ & \textbf{69.08}  & \textbf{74.88}  & \textbf{69.76}  & \textbf{78.15}  & \textbf{83.40}  & \textbf{82.74}  & \textbf{73.10}  & \textbf{67.06}  \\ \hline
        Channel & 63.51  & 68.66  & 66.13  & 71.70  & - & - & - & - \\ 
        Channel$\dagger$   & 65.71  & 72.03  & 67.74  & \textbf{76.91}  & - & - & - & - \\
        Channel$\ddagger$ & \textbf{68.69}  & \textbf{75.42}  & \textbf{69.89}  & 75.67  & - & - & - & - \\ \hline
        ML-LSTM & 63.94 & 69.02 & 62.90 & 64.73 &82.77 & 82.05& 	70.68 &	65.64 \\
        ML-LSTM $\dagger$ & 67.80 &	73.24 &	65.05 &	71.59 & 83.44&  82.80 &73.67 &	67.77  \\
        ML-LSTM $\ddagger$& \textbf{70.70} &	\textbf{77.00} &	\textbf{67.47} &	\textbf{73.50} & \textbf{83.89} &	\textbf{83.12} &	\textbf{74.45} & \textbf{69.41}   \\ \hline
        Sum & 64.90  & 71.03  & 63.44  & 69.08  & 82.94  & 82.47  & 73.00  & 66.61  \\ 
        Sum$\dagger$  & 66.52  & 72.98  & 66.80  & 73.14  & \textbf{83.73}  & 83.07  & \textbf{73.87}  & \textbf{66.84}  \\
        Sum$\ddagger$ & \textbf{69.93}  & \textbf{76.26}  & \textbf{69.49 } & \textbf{74.80}  & 83.69  & \textbf{83.14}  & 73.38  & 66.55  \\ \midrule
        Weight & 65.33  & 71.10  & 66.53  & 73.26  & 82.65  & 82.19  & 72.42  & 65.16  \\ 
        Weight$\dagger$  & 66.64  & 72.88  & 68.41  & \textbf{77.22}  & 83.23  & 82.59  & 73.48  & 67.40  \\ 
        Weight$\ddagger$ & \textbf{68.46}  & \textbf{74.62}  & \textbf{71.10 } & 76.73  & \textbf{83.98}  & \textbf{83.42}  & \textbf{74.35}  & \textbf{69.52} \\ 
        \Xcline{1-9}{0.7pt}
    \end{tabular}
\end{table*}
\noindent\textbf{Comparision in Complex Transformer-based Architecture.} Currently, numerous works in multi-modal learning are built upon a unified multi-modal transformer architecture. To assess the effectiveness of the AMSS/AMSS+ method within this framework, we investigate alternative backbone networks. While handling audio-video datasets, in addition to utilizing CNN as the backbone, we introduce a fusion architecture based on Transformer, namely MBT~\cite{NagraniYAJSS21}. This methodology comprises cross-modal interaction layers utilizing bottlenecks to integrate information between modalities. We employ two training approaches: training from scratch (scratch) and pre-training using ImageNet/Audioset for the ViT/AST models (pretrain) in MBT. 
Based on the results shown in Table~\ref{tab:backbone}, the following conclusions can be observed. (1) The effectiveness of modality imbalance methods on this architecture is limited compared to CNN architecture. In complex cross-modal interaction scenarios, certain modality imbalance methods prove ineffective. For example, OGM-GE and PMR, regardless of whether they are employed in a from-scratch training or pre-training setting, exhibit performance even worse than the benchmark results of MBT. (2) Whether employing a CNN architecture or a sophisticated multi-modal Transformer architecture, the AMSS+ strategy consistently maintains superior performance across almost all metrics. This demonstrates that our method possesses excellent adaptability. (3) Whether the model is pre-training or not does not impact the performance of our method. This flexibility allows our approach to be applied seamlessly across various scenarios. (4) AMSS+ continues to outperform AMSS, aligning once again with our theoretical expectations. 


\noindent\textbf{Exploration on Different Fusion Strategy.} 
In the aforementioned context, we noted that traditional fusion methods are affected by the challenge of modality imbalance compared to multi-modal rebalancing strategies. The problem of modality imbalance constrains the performance of these methods. Therefore, we investigate the efficacy of integrating AMSS/AMSS+ with various fusion techniques to tackle the challenge of modality imbalance under different fusion strategies. 
The AMSS/AMSS+ strategy is applied to five fusion methods: Affine, Channel, ML-LSTM, Sum, and Weight, all of which have been introduced previously. It is noteworthy that Sum and Weight are prediction-level fusion methods, Channel is hybrid fusion method, while the others operate at the feature level. Fusion strategies at different levels correspond to different methods for the selection scope of subnetworks for each modal parameter. The relevant content is introduced in implementation details. As shown in Table~\ref{tab:various_method}, the combination of AMSS/AMSS+ with either feature-level or prediction-level fusion methods reveals a significant enhancement in their performance, underscoring the effectiveness of the AMSS strategy in augmenting their capabilities and mitigating the problem of modality imbalance across diverse fusion strategies. Compared to AMSS, the rebalancing strategy of AMSS+ demonstrates superior performance improvement, except for the Sum strategy on the text-image dataset.
\begin{table*}[tbp]
    \centering
    \caption{Performance of diverse sampling strategies. Baseline means no subnetwork optimization strategy. The best results are highlighted in \textbf{bold}. 
    } \vspace{-5pt}
    \label{tab:sample}
    \normalsize
    \begin{tabular}{l|cc|cc|cc|cc}
    \Xcline{1-9}{0.7pt}
        \multicolumn{1}{l|}{\multirow{2}{*}{Methods}} &
          \multicolumn{2}{c|}{{Kinetics-Sound}} &
          \multicolumn{2}{c|}{{CREMA-D}} &
          \multicolumn{2}{c|}{{Sarcasm-Detection}} &
          \multicolumn{2}{c}{{Twitter-15}} \\
        & ACC & mAP & ACC & mAP & ACC & Mac-F1 & ACC & Mac-F1  \\ \hline
        Baseline & 64.55  & 71.30  & 63.31  & 68.41  & 82.86  & 82.40  & 70.11  & 63.86  \\ 
        Random & 66.45  & 72.43  & 65.59  & 72.05  & 83.60  & 83.00  & 74.35  & 68.87    \\ 
        AMSS & 68.96  & 74.89  & 67.61  & 73.97  & 84.14  & 83.69  & \textbf{75.89}  & \textbf{69.81}   \\ 
        AMSS+ & \textbf{72.25}  & \textbf{79.13}  & \textbf{70.16}  & \textbf{76.14}  & \textbf{84.35}  & \textbf{83.77}  & 75.12  & 69.23  \\ \Xcline{1-9}{0.7pt}
    \end{tabular}\vspace{-8pt}
\end{table*}
\subsection{Ablation Study}
\subsubsection{Sampling Mechanism Variant}
In this section, our objective is to delve into the significance of multi-modal subnetworks derived from non-uniform adaptive sampling techniques. To achieve this, we employ the uniform sampling mechanism (Random) to conduct ablation studies for AMSS+. In the Random method,  similar to that in preliminary experiments, each mask unit is sampled with an identical probability. The difference lies in how the update ratios $\rho$ for various modalities depend on the dynamic update strategy of AMSS+. As shown in Table \ref{tab:sample}, the Random method selects a multi-modal subnetwork that exhibits superior performance compared to the Baseline, highlighting the effectiveness of updating each modality with distinct update ratios depending on AMSS+. Across both biased and unbiased estimation scenarios, the utilization of non-uniform adaptive sampling showcases significant superiority over uniform sampling in performance metrics.
Based on these observations, we can confidently conclude that AMSS+ with non-uniform adaptive sampling is indeed more effective for downstream tasks.

\begin{figure}[tbp]
  \centering
  \includegraphics[width=0.244\textwidth]{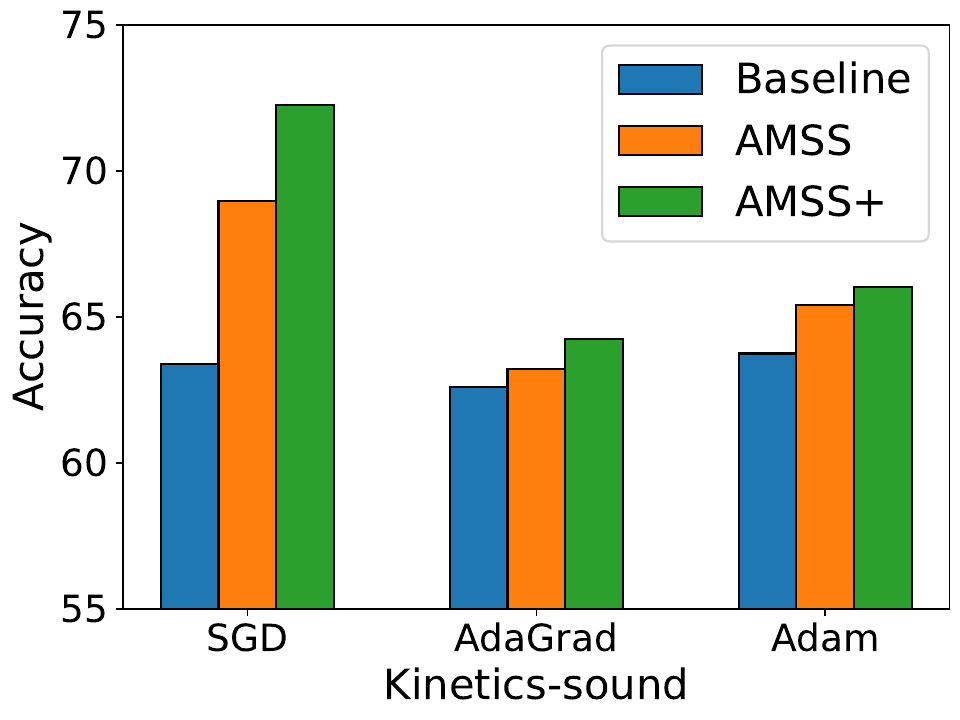}\hfill
  \includegraphics[width=0.244\textwidth]{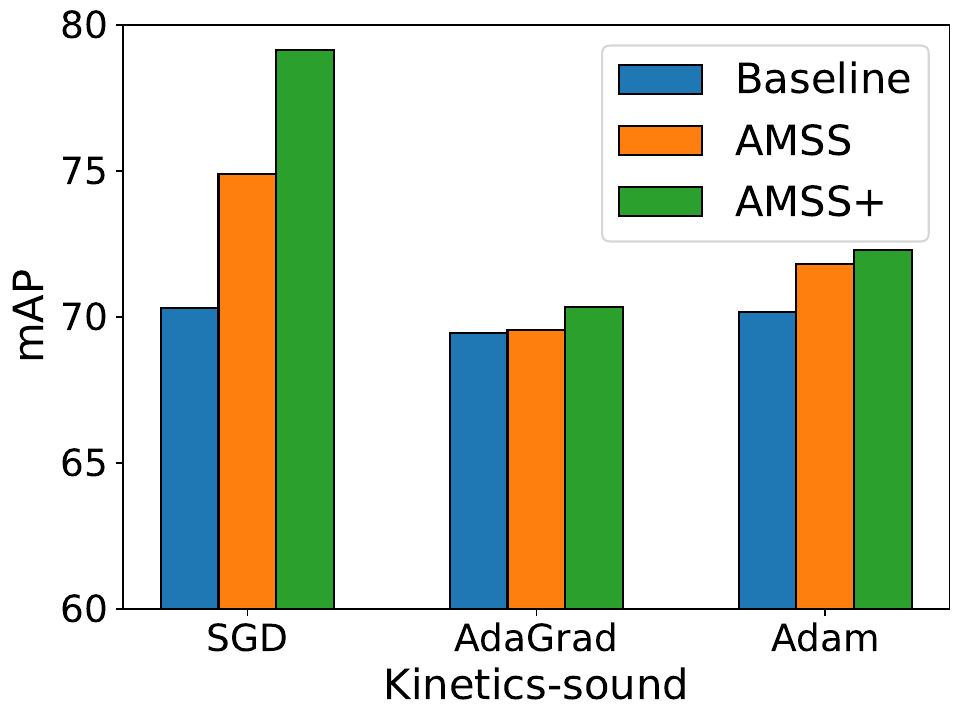}\hfill
  \includegraphics[width=0.244\textwidth]{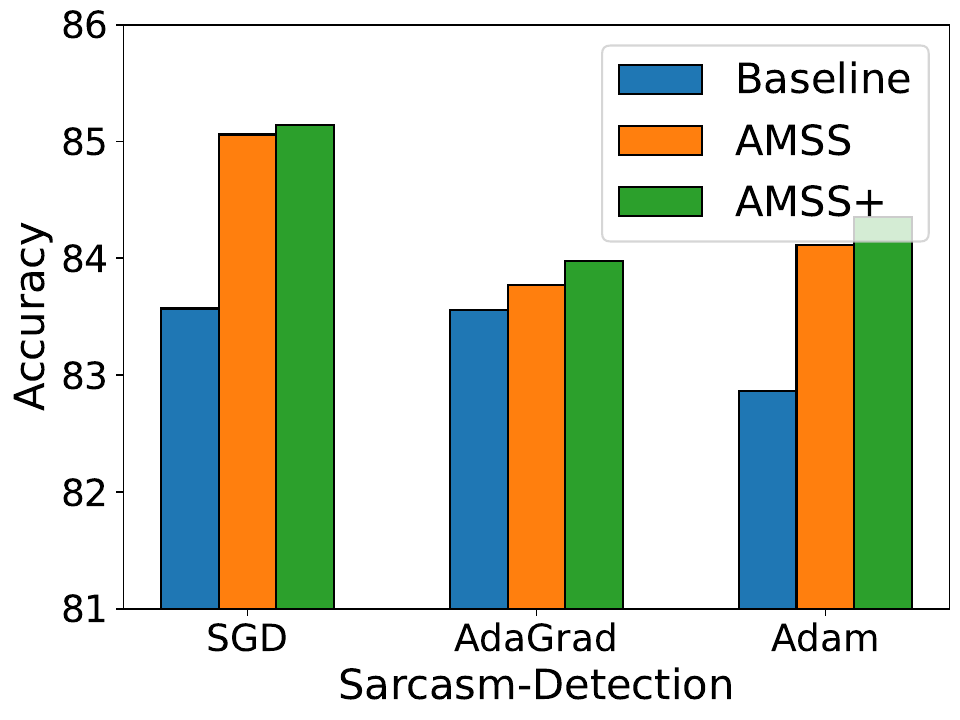}\hfill
  \includegraphics[width=0.244\textwidth]{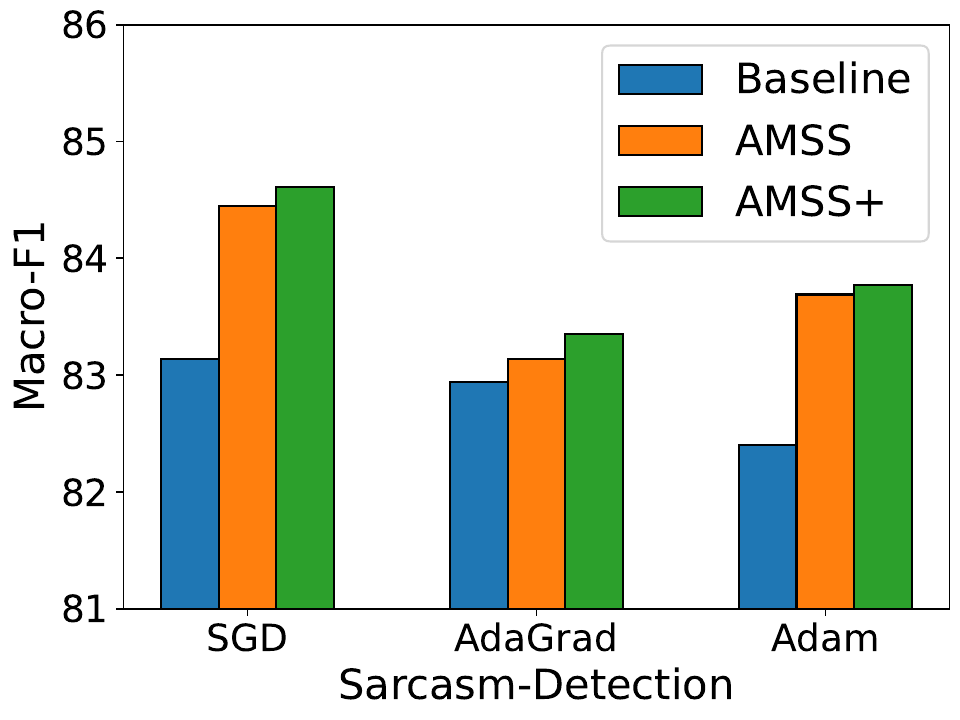}
  \caption{Experiments with SGD, AdaGrad and Adam optimizers in Kinetics-Sound and Sarcasm-Detection. Baseline means no extra modulation.}\vspace{-5pt}
  \label{fig:optimizer}
\end{figure}


\begin{table}[tbp]
    \centering
    \caption{The effectiveness of each component.}\vspace{-10pt}
    \label{tab:component}
    \Large
    \begin{center}
    \resizebox{1\columnwidth}{!}{
    
        \begin{tabular}{l|cc|cc}
            \Xcline{1-5}{0.7pt}
            \multicolumn{1}{c|}{\multirow{2}{*}{Methods}} &
          \multicolumn{2}{c|}{{Kinetics-Sound}} &
          \multicolumn{2}{c}{{Sarcasm-Detection}} \\
            & ACC & mAP & ACC & Mac-F1 \\ \hline
            Baseline & 64.55  & 71.30  & 82.86  & 82.40  \\
            ~~w/ Classifier Mask & 65.37  & 70.87  & 83.40  & 83.05  \\ 
            ~~w/ Backbone Mask& 66.68  & 72.23  & 83.81  & 83.17  \\ 
            ~~AMSS & \textbf{68.96}  & \textbf{74.89}  & \textbf{84.14}  & \textbf{83.69}  \\ \hline
            ~~w/ Classifier Mask+ & 66.80  & 73.80  & 83.44  & 82.93  \\ 
            ~~w/ Backbone Mask+ & 69.15  & 76.13  & 83.77  & 83.10  \\ 
            ~~AMSS+ & \textbf{72.25}  & \textbf{79.13}  & \textbf{84.35}  & \textbf{83.77} \\ 
            \Xcline{1-5}{0.7pt}
        \end{tabular} 
        }
    \end{center}\vspace{-10pt}
\end{table}

\subsubsection{Different Optimizer}
Our theoretical analysis is grounded in the SGD optimizer. To further substantiate the adaptability of our approach when integrated with a variety of optimizers, AdaGrad~\cite{duchi2011adaptive} and Adam~\cite{kingma2014adam} optimizers are also employed in our experimental validation. These optimizers are applied to two distinct types of datasets to ensure a comprehensive evaluation of our approach's performance. Subsequently, we integrate the AMSS or AMSS+ strategy and assess its performance across different optimizers. By incorporating AMSS or AMSS+, we aim to enhance the performance of our method across a spectrum of optimization strategies. The results, depicted in Figure~\ref{fig:optimizer}, emphasize the diverse performance exhibited by the selected optimizers on Kinetics-Sound and Sarcasm-Detection datasets. Significantly, our method consistently showcases exceptional adaptability, consistently surpassing baseline results and achieving substantial performance improvements. 
Furthermore, compared to other optimizers, the improvement in the effectiveness of our method is more pronounced under the SGD optimizer. This is attributed to the adaptive adjustment of learning parameters inherent to the optimizer itself.
This sustained success across different optimizers underscores the adaptability of our approach in optimizing model performance, irrespective of the underlying optimizer. 
\begin{figure}[tbp]
  \centering
  \includegraphics[width=0.48\textwidth]{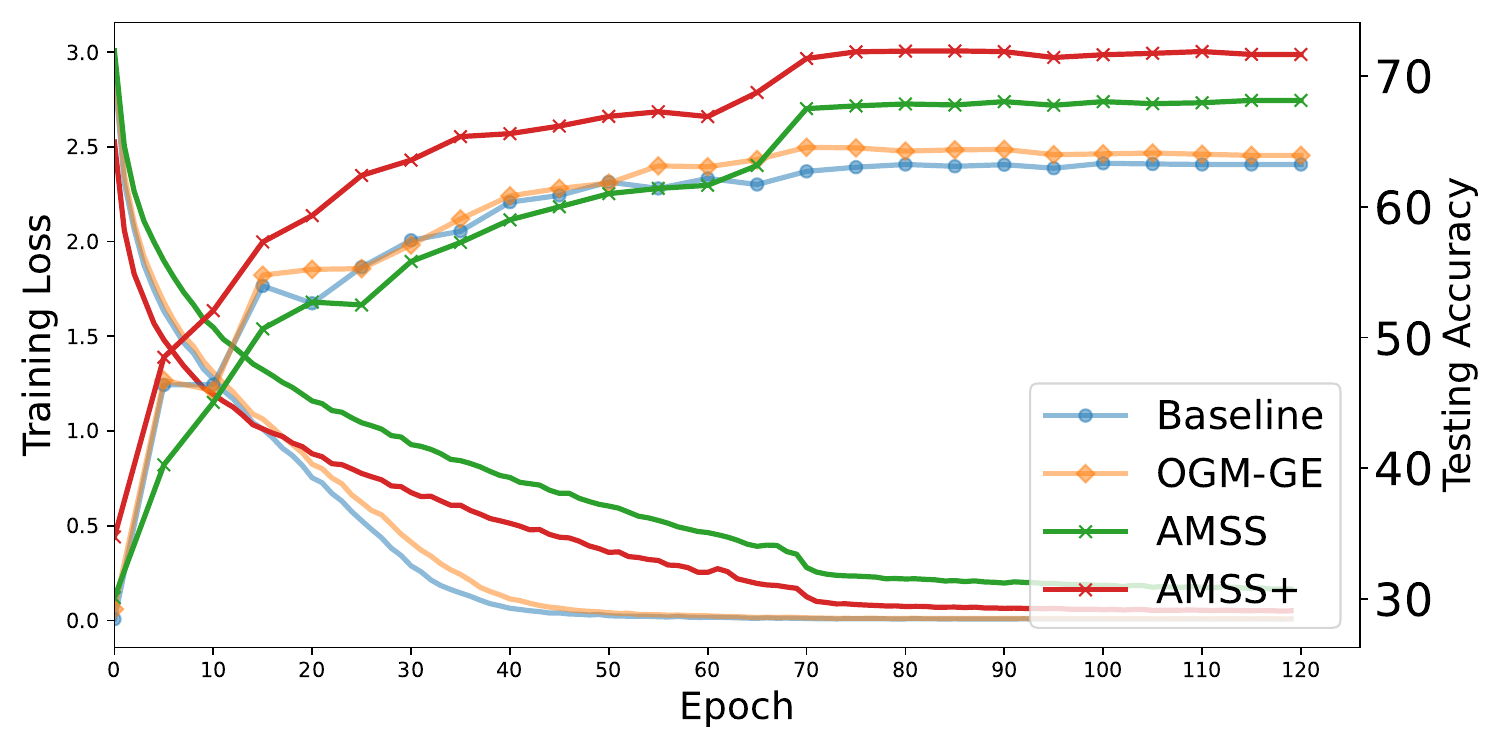}\hfill
  \vspace{-5pt}
  \caption{On the Kinetics-Sound dataset, we employ the concatenation fusion method for the joint training of multi-modal models, encompassing Baseline, OGM-GE, AMSS, and AMSS+. We investigate the changes in training loss and evaluate the variations in test performance across these multi-modal models. Baseline means no gradient modulation strategy.}\vspace{-15pt}
  \label{fig:convergence}
\end{figure}

\begin{figure*}[tbp]
\centering
  \subfloat[Kinetics-Sound]{
  \includegraphics[width=0.24\textwidth]{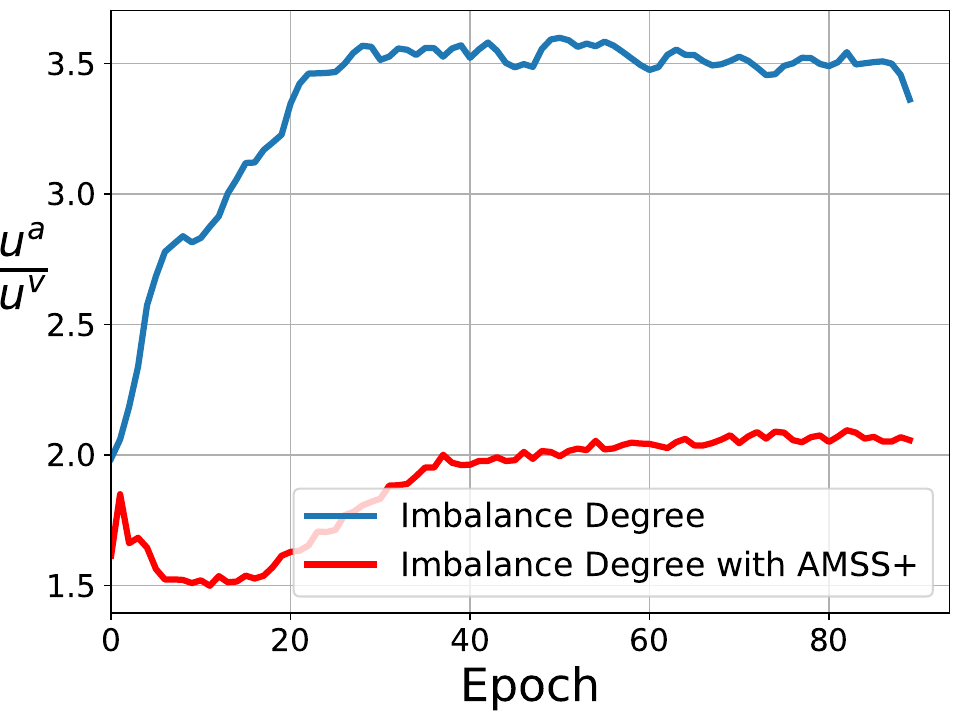}\hfill
  \includegraphics[width=0.24\textwidth]{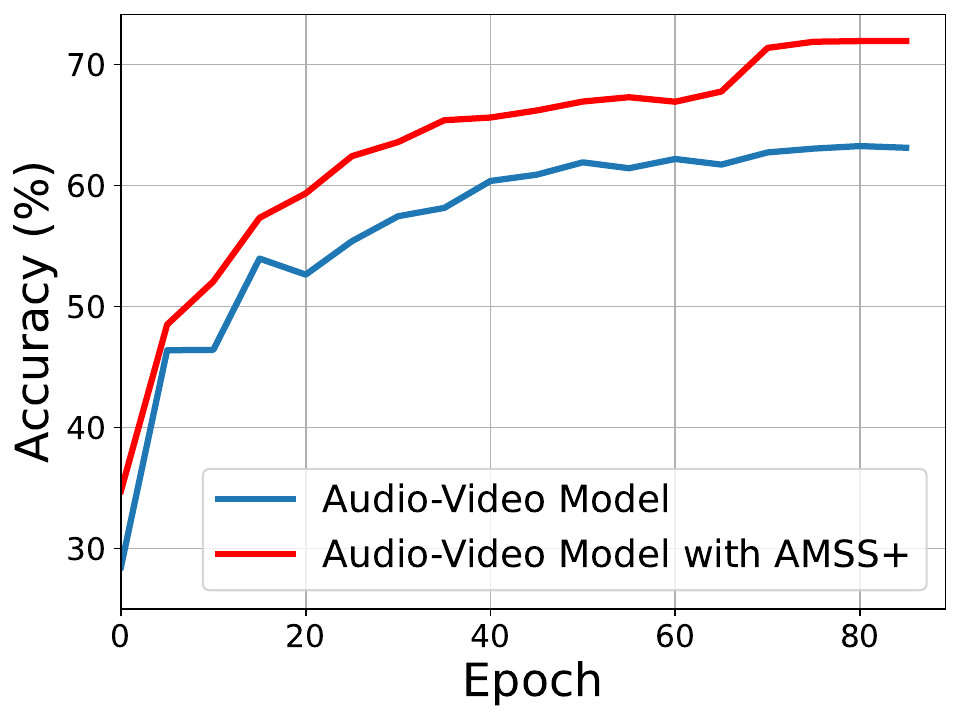}\hfill
  \includegraphics[width=0.24\textwidth]{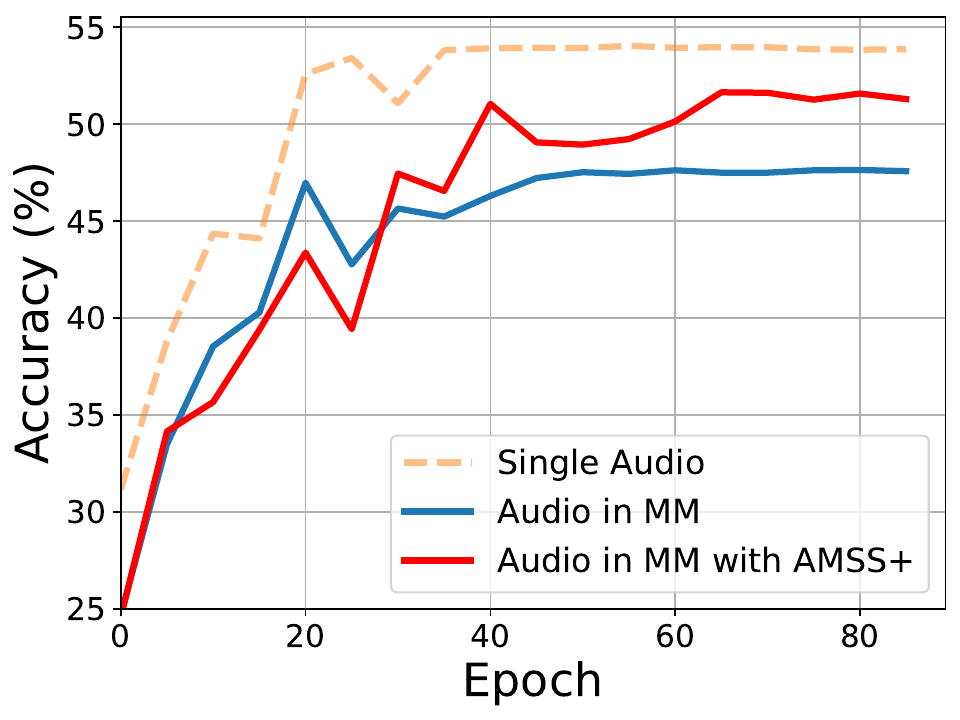}\hfill
  \includegraphics[width=0.24\textwidth]{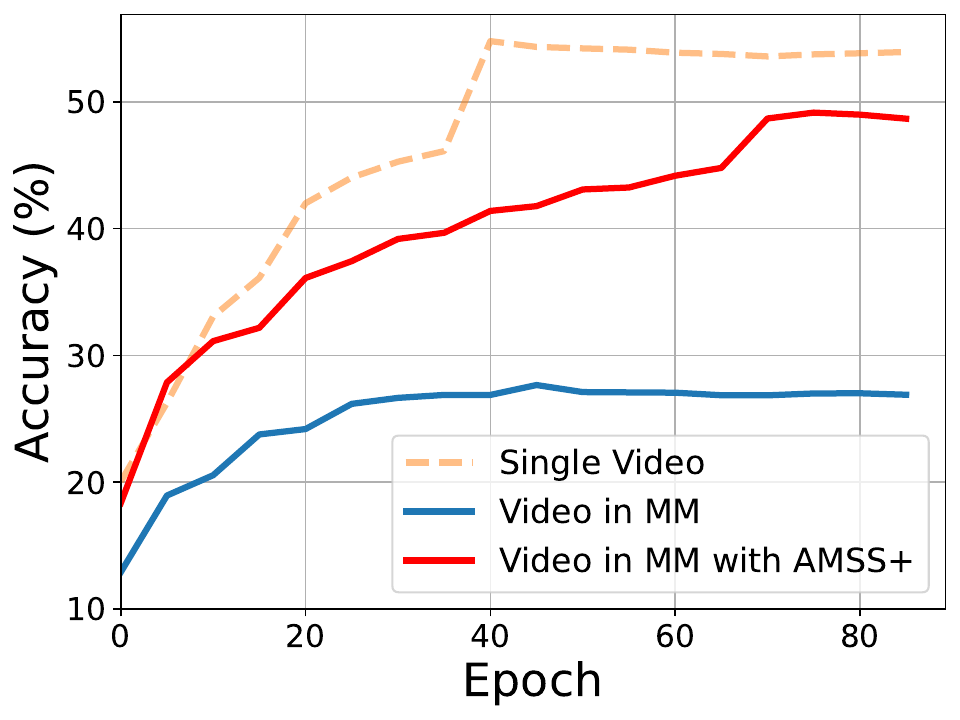}\hfill
  } \\
  \subfloat[Sarcasm-Detection]{
  \includegraphics[width=0.24\textwidth]{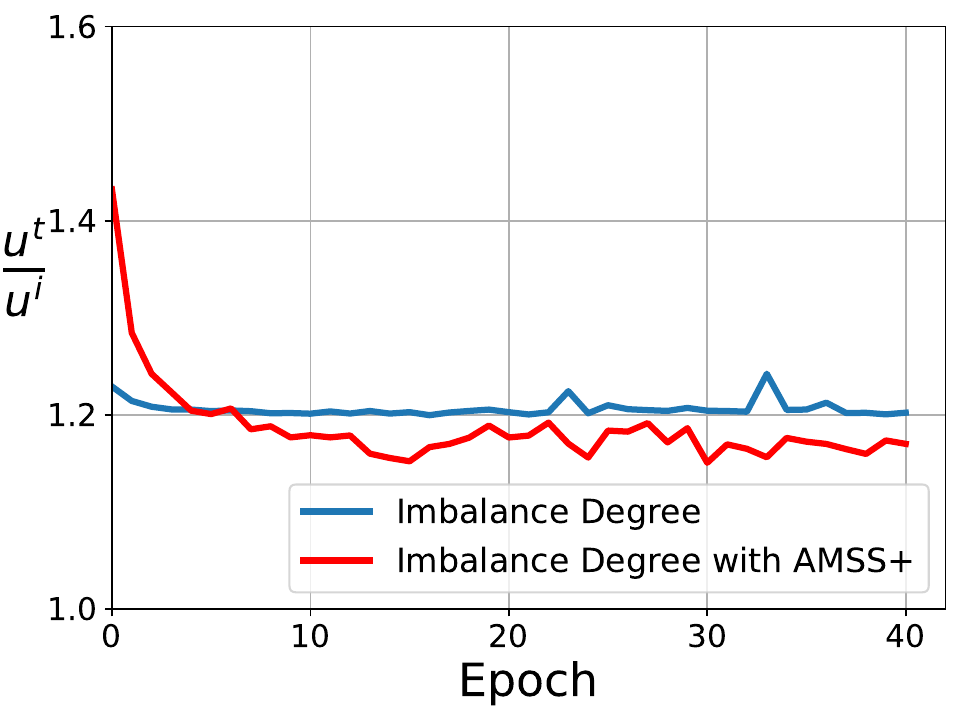}\hfill
  \includegraphics[width=0.24\textwidth]{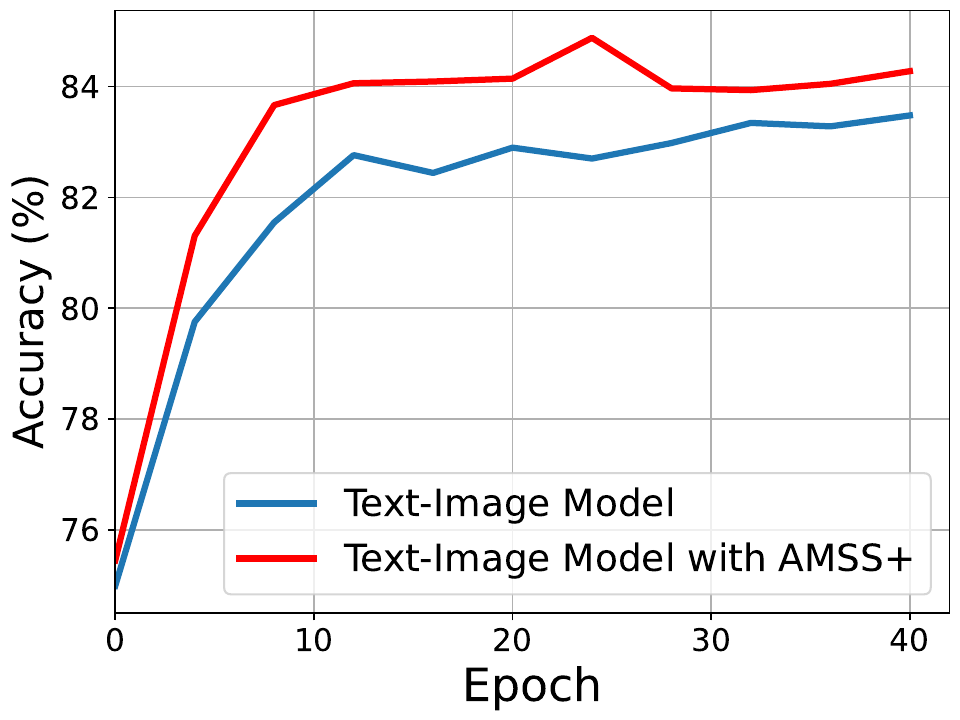}\hfill
  \includegraphics[width=0.24\textwidth]{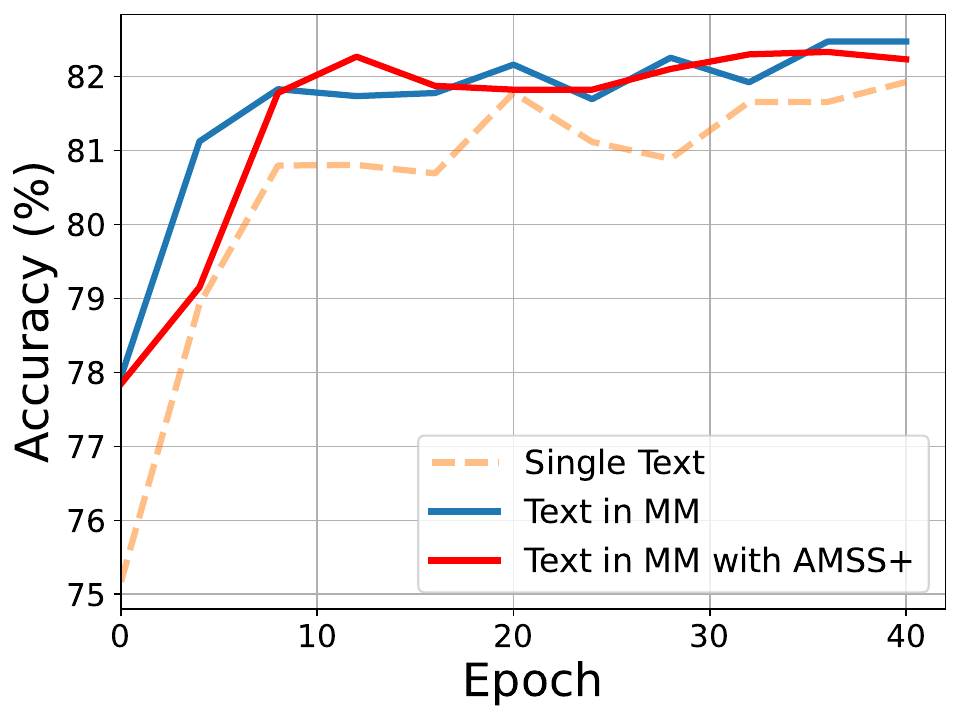}\hfill
  \includegraphics[width=0.24\textwidth]{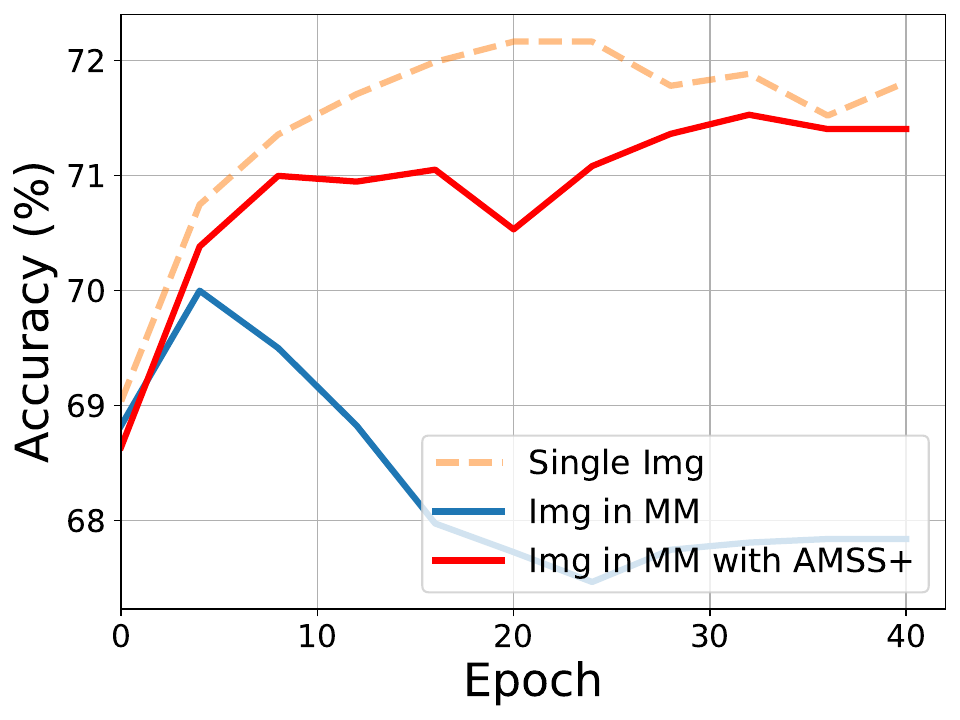}\hfill
  }
\vspace{-5pt}
  \caption{Analysis of Modality Imbalance Problem. Each dataset is represented in Figures from left to right, depicting the variation in model imbalance degree, the comparison between our method and the Baseline model, the performance of single-modal branches in multi-modal trained models, and the performance of single-modal branches with AMSS+, including Audio/Text and Video/Img modalities. The fusion method used in the multi-modal model is Concat.} 
  \label{fig:imbalance}
  \vspace{-10pt}
\end{figure*}
\subsubsection{Component Ablation Analysis}
In both fusion methods at the prediction level and the feature concatenation of a single fully connected layer (considered decoupled)~\cite{peng2022balanced}, we implement the adaptive subnetwork masking strategy independently for the classifier and the backbone network. To assess the significance of each module within the model, we perform an extensive module ablation analysis on two different types of datasets. Table~\ref{tab:component} illustrates the impact of each module on our method. We present four model variants for examination: Baselines; Model with classifier masks/masks+, excluding backbone masks/masks+; Model with backbone masks/masks+, excluding classifier masks/masks+; The proposed AMSS/AMSS+. In our analysis, each module integrated into our method consistently demonstrates a contribution to the overall performance enhancement of the model. Compared to applying the masking strategy to the classifier, adopting the masking strategy for the backbone network leads to a more pronounced improvement in the model's performance.


\subsection{Optimization Analysis}
In our research, we scrutinize the impact of diverse parameter update strategies on model training, concentrating on a comparative study of four distinct methods employed for multi-modal joint learning in the Kinetics-Sound dataset. 
Figure~\ref{fig:convergence} demonstrates that the implementation of Mask Subnetwork (AMSS, AMSS+) strategies appears to moderate the pace at which the model learns, in contrast to the baseline model.
From the theoretical perspective, subnetwork optimization strategies are expected to moderately slow the model's rate of convergence, a trend that is consistent with the trajectory of training loss depicted in  Figure~\ref{fig:convergence}. 
It is noteworthy that despite this deceleration, such strategies maintain robust performance levels and demonstrate superior generalization abilities. In contrast, while other methods may exhibit a faster modal convergence rate, they evidently succumb to model overfitting issues, lacking generalization abilities. 
Additionally, the AMSS+ strategy, which provides an unbiased approach to subnetwork optimization, surpasses its biased counterpart, showing improved convergence rates and overall performance. In conclusion, the adoption of adaptive strategies for updating subnetworks offers tangible benefits.
\begin{table*}
\caption{
Accuracy and gradient update ratios across two modalities under varying \(\tau\). }
\vspace{-5pt}
\label{tab:parameter}
\centering
\renewcommand{\arraystretch}{1.5}
\centering
\begin{tabular}{c|cccccccccc|ccc}
\Xcline{1-14}{0.7pt}
    \multirow{2}{*}{AMSS+} & \multicolumn{10}{c|}{\multirow{1}{*}{$\tau<1$}} & \multicolumn{3}{c}{\multirow{1}{*}{$\tau\geq1$}} \\
    & 0.10 & 0.20 & 0.25 & 0.30 & 0.40 & 0.50 & 0.60 & 0.70 & 0.80 & 0.90 & 1.00 & 2.00 & 4.00 \\ \hline
    ACC        & 72.17 & \textbf{72.52} & 72.25 & 72.13 & 71.82 & 71.47 & 71.67 & 71.05 & 71.01 & 70.85 & 70.20 & 69.85 & 69.69 \\ 
    {$\overline{\rho}^{(1)}$} & 0.2238 & 0.2436 & 0.2479 & 0.2508 & 0.2684 & 0.2774 & 0.2914 & 0.2940 & 0.3105 & 0.3156 & 0.3305 & 0.3964 & 0.4404 \\ 
    {$\overline{\rho}^{(2)}$} & 0.7762 & 0.7564 & 0.7521 & 0.7492 & 0.7316 & 0.7226 & 0.7086 & 0.7060 & 0.6895 & 0.6844 & 0.6695 & 0.6036 & 0.5596 \\ 
\Xcline{1-14}{0.7pt}
\end{tabular}\vspace{-12pt}
\end{table*}
\vspace{-5pt}
\subsection{Analysis of Modality Imbalance}
The issue of modality imbalance highlights a challenge in the process of multi-modal learning, where the dominance of one modality inhibits the exploration of distinctive features in other modalities. To further scrutinize the impact of our approach on model optimization, we initially define the degree of modality imbalance, denoted as $\frac{u^{(1)}}{u^{(2)}}$, utilizing the modal significance derived from Equation~\ref{con:contribution_u}. We use $\frac{u^{a}}{u^{v}}$ and $\frac{u^{t}}{u^{i}}$ to represent the imbalance degree of the audio-video and text-image modalities, respectively. Subsequently, we analyze the performance of the multi-modal model both before and after the application of the AMSS+ strategy, along with an examination of the performance of uni-modal branches within the multi-modal model. 
The observations and conclusions from the imbalance analysis are presented in Figure~\ref{fig:imbalance}. Due to space limitations, the results of CREMA-D and Twitter-15 datasets are presented in the Appendix. (1) In vanilla multi-modal learning, the Audio/Text branch closely approximates the performance of the single Audio/Text model, while the Video/Img branch exhibits minimal effective learning, resulting in a noticeable gap compared to the single Video/Img modality. As depicted in the curve of imbalance degree variations, our method effectively alleviates the modality imbalance to varying degrees across the two datasets. This serves as concrete evidence of the efficacy of our approach. (2) The modality imbalance issue is most severe in the Kinetics-Sound, leading to under-utilization of the video modality in joint training. The performance of the multi-modal model is primarily derived from the Audio modality branch. The AMSS+ strategy allows for the full utilization of the Video modality, bringing its performance close to single-modal training. 
(3) On the text-img dataset, where the modality imbalance issue is less severe than in audio-video, the model is still able to address this problem, leading to improvements in the performance of non-dominant modalities (image) and ultimately achieving better overall model performance. 


  
  

\subsection{Exploring the Gradient Update Ratio}
AMSS/AMSS+ introduces one hyper-parameter \(\tau\) in Equation \ref{con:mask_ratio}. In this section, we vary \(\tau\) within the set \{0.1, 0.2, 0.25, 0.3, 0.4, 0.5, 0.6, 0.7, 0.8, 0.9, 1, 2, 4\} on Kinetics-sound to explore its impact on model performance and the variations in gradient update ratios across two modalities. Since the gradient update ratio in our method is on the batch level, we calculate the average gradient update ratio for each modality across all iterations within the initial 10 epochs, denoted here as \(\overline{\rho}\). \(\overline{\rho}^{(1)}\)/\(\overline{\rho}^{(2)}\) represent the Audio/Video modalities.
As illustrated in Table~\ref{tab:parameter}, we observe that even when \(\tau = 1\), indicating the absence of \(\tau\), the performance of AMSS+ still surpasses the results of the state-of-the-art method (67.10). Moreover, when \(\tau < 1\), there is a pronounced size difference in the gradient update ratio among modalities, facilitating a better balance in optimization across modalities, thereby enhancing the overall efficacy of the model. Conversely, when \(\tau > 1\), this size difference is markedly reduced, and the issue of modality imbalance remains significant, thus limiting the performance of the model.
Intriguingly, it is observed that the size difference between modalities should not be excessively large as it could detrimentally impact the model's performance, potentially by inhibiting the update momentum of the dominant modality, as exemplified when $\tau = 0.1$.
Moreover, through meticulous hyper-parameter modulation, AMSS+ demonstrates an ability to surpass previous benchmarks, achieving enhanced outcomes. For example, at \(\tau=0.2\), AMSS+ exhibits a performance increment of 0.27 compared to our previous results in the Kinetics-sound dataset. 
Additionally, when fixing ${\overline{\rho}}^{(1)}=1$ and $\overline{\rho}^{(2)}=1$, i.e., not masking the parameters, and scaling the parameters only with $\frac{1}{p_j+L_i}$ in AMSS+, the model performance is 69.39. This indicates that the masking subnetwork strategy in AMSS+ is crucial.
%

\vspace{-5pt}
\section{Conclusion}
In our study, we found that the strategy of employing uniformly weighted modulation parameters for gradient modifications is sub-optimal for resolving the challenges posed by modality imbalance. To address this issue more effectively, we introduce the innovative approach of Adaptively Mask Subnetworks considering Modal Significance (AMSS). Our approach emphasizes performing gradient updates on differently sized, more promising subnetworks, selected adaptively for each modality, based on the batch-level mutual information rate. We conduct a theoretical analysis of effectiveness of this strategy and further propose an unbiased estimation optimization strategy, AMSS+. Furthermore, our approach can serve as a flexible plug-in strategy for existing multi-modal models.


%

%
%

%

\ifCLASSOPTIONcaptionsoff
  \newpage
\fi

\bibliographystyle{IEEEtran}{
\bibliography{IEEEabrv, AMSS}}
\vspace{-10pt}
\begin{IEEEbiography}
[{\includegraphics[width=1in,height=1.25in,clip,keepaspectratio]{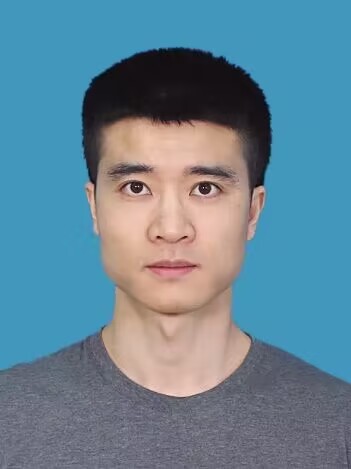}}]{Yang Yang}
received the Ph.D. degree in computer science, Nanjing University, China in 2019. At the same year, he became a faculty member at Nanjing University of Science and Technology, China. He is currently a Professor with the school of Computer Science and Engineering. His research interests lie primarily in machine learning and data mining, including heterogeneous learning, model reuse, and incremental mining.  He has published prolifically in refereed journals and conference proceedings, including IEEE Transactions on Knowledge and Data Engineering (TKDE), ACM Transactions on Information Systems (ACM TOIS), ACM Transactions on Knowledge Discovery from Data (TKDD), ACM SIGKDD, ACM SIGIR, WWW, IJCAI, and AAAI. He was the recipient of the the Best Paper Award of ACML-2017. He serves as PC in leading conferences such as IJCAI, AAAI, ICML, NeurIPS, etc.
\end{IEEEbiography}

\begin{IEEEbiography}[{\includegraphics[width=1in,height=1.25in,clip,keepaspectratio]{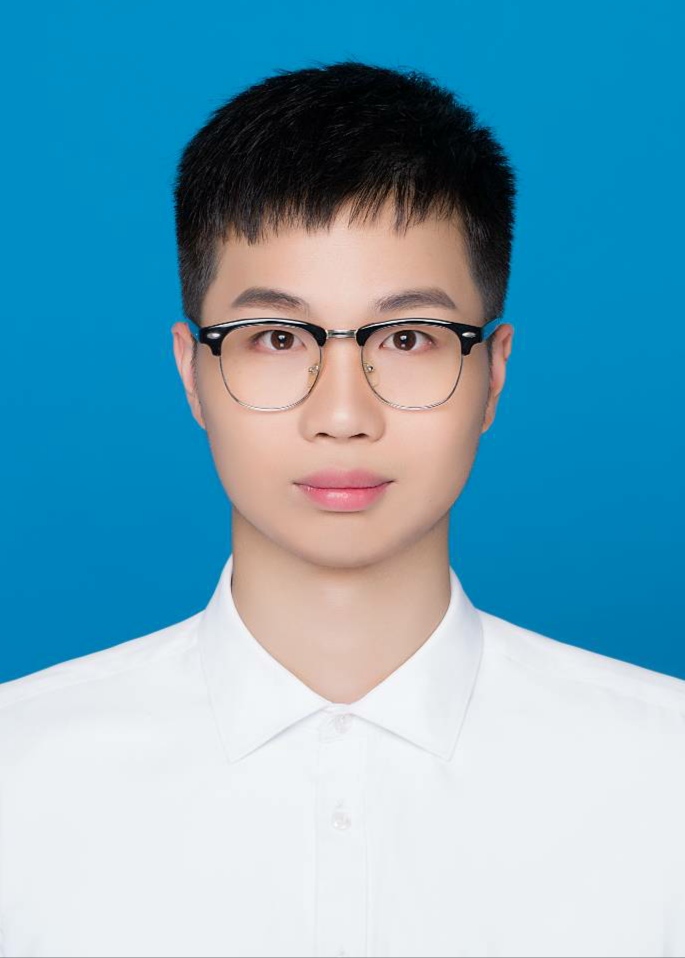}}]{Hongpeng Pan}
is currently working towards the M.S. degree at the School of Computer Science and Engineering, Nanjing University of Science and Technology. His research interests mainly lie in deep learning and data mining. He is currently focusing on multi-modal learning.
\end{IEEEbiography}

\begin{IEEEbiography}[{\includegraphics[width=1in,height=1.25in,clip,keepaspectratio]{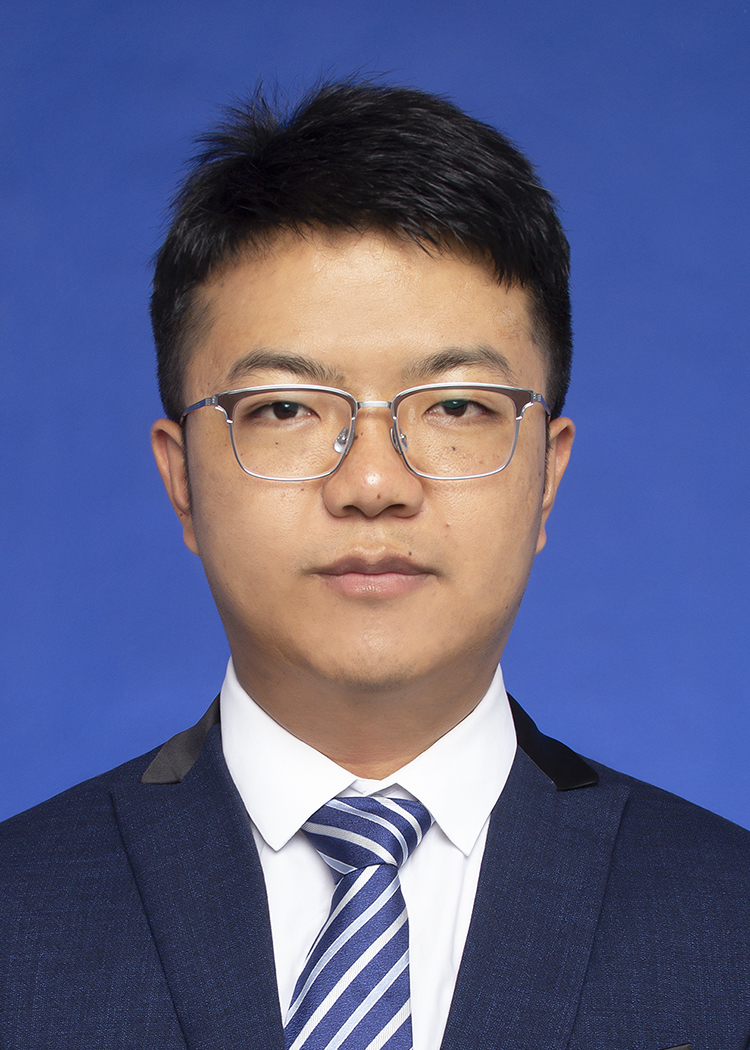}}]{Qing-Yuan Jiang}
received the BSc and the PhD degrees in computer science from Nanjing University, China. He has publised 9 papers in leading international journals/conferences. He serves as a PC member in leading conferences such as AAAI, IJCAI, etc. He is currently a researcher engineer at Huawei. His research interests are in learning to hash and multi-modal retrieval.  
\end{IEEEbiography}

\begin{IEEEbiography}[{\includegraphics[width=1in,height=1.25in,clip,keepaspectratio]{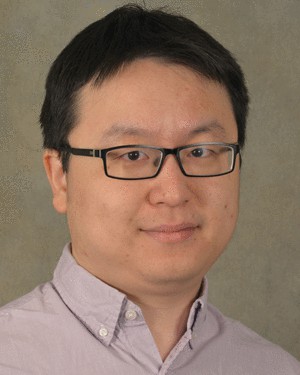}}]{Yi Xu}
    received the PhD degree in computer science from the University of Iowa, Iowa City, Iowa USA, in 2019. He is currently a professor with the School of Control Science and Engineering, Dalian University of Technology, China. His research interests are machine learning, optimization, deep learning, and statistical learning theory. He has published more than twenty papers in refereed journals and conference proceedings, including IEEE Transactions on Pattern Analysis and Machine Intelligence (TPAMI), Transactions on Machine Learning Research, NeurIPS, ICML, ICLR, CVPR, AAAI, IJCAI, UAI. He serves as a SPC/PC/Reviewer in leading conferences such as NeurIPS, ICML, ICLR, CVPR, AAAI, IJCAI, etc.   
\end{IEEEbiography}

\begin{IEEEbiography}[{\includegraphics[width=1in,height=1.25in,clip,keepaspectratio]{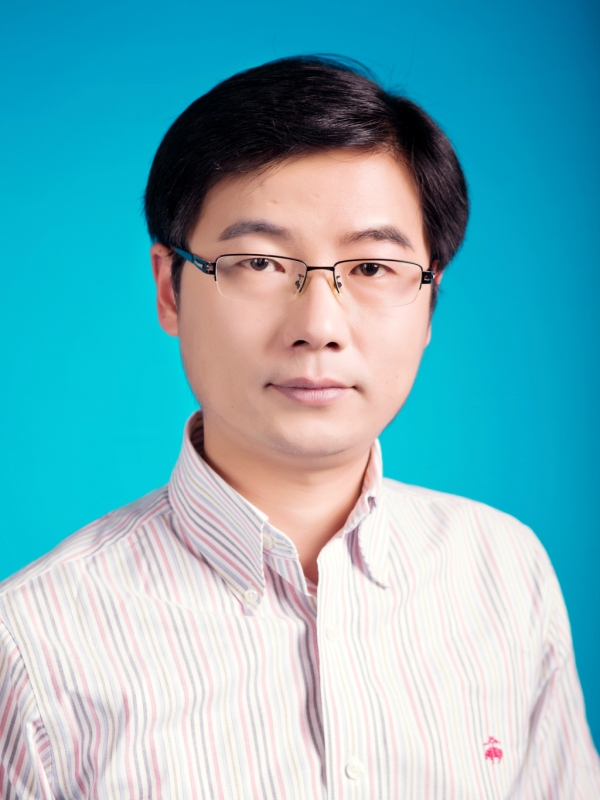}}]
{Jinhui Tang} (Senior Member, IEEE) received the B.E. and Ph.D. degrees from the University of Science and Technology of China, Hefei, China, in 2003 and 2008, respectively. He is currently a Professor with the Nanjing University of Science and Technology, Nanjing, China. He has authored more than 200 articles in top-tier journals and conferences. His research interests include multimedia analysis and computer vision. Dr. Tang was a recipient of  the Best Paper Awards in ACM MM 2007, PCM 2011, ICIMCS 2011, and ACM MM Asia 2020, the Best Paper Runner-Up in ACM MM 2015, and the Best Student Paper Awards in MMM 2016 and ICIMCS 2017. He has served as an Associate Editor for the IEEE TMM, IEEE TNNLS, the IEEE TKDE, and the IEEE TCSVT. He is a Fellow of IAPR. 
\end{IEEEbiography}



\onecolumn 
\begin{center}\Large
{\bf Appendix}
\end{center}

\renewcommand{\thefigure}{\Alph{figure}}
\renewcommand{\thesection}{\Alph{section}}

\setcounter{figure}{0}
\setcounter{section}{0}

\section{Convergence Analysis}In this subsection, we provide a convergence analysis for the proposed method under the non-convex optimization setting. The detailed process of the proof is as follows.
\subsection{Case I: Biased Stochastic Gradient}
For the first case of biased stochastic gradient, recall that the update step of stochastic gradient descent (SGD) is 
\begin{align}\label{app:eqn:updt:sgd}
{\bf w}(t+1)={\bf w}(t)-\eta \nabla \ell({\bf w}(t)) \odot {\bf m}{(t)}
\end{align}
where $\nabla \ell({\bf w}(t))$ is the stochastic version of the gradient of loss function $\nabla\mathcal{L}({\bf w}(t))$ at $\theta_t$ and $\eta>0$ is the learning rate, and the element of ${\bf m}{(t)}$ is given by
\begin{equation}
m_j{(t)}= \left\{\begin{array}{ll}
1,  & \text{if}\;\;w_j(t) \in {{\mathcal{S}}(t)}, \\
0,  & \text{otherwise},
\end{array}\right.
\end{equation}
To analyze the convergence rate, we make the following common assumptions for the loss function. 
\begin{assumption}[Smoothness]\label{ass:smooth}
We assume that the loss function $\mathcal{L}$ is $L$-smooth. That is,
for any ${\bf w}, {\bf w}'$, we have
\begin{align}\label{ass:smotth:eqn}
\mathcal{L}\left({\bf w}\right) - \mathcal{L}\left({\bf w}'\right) 
 \le\langle \nabla \mathcal{L}\left({\bf w}'\right), {\bf w} - {\bf w}'\rangle + \frac{L}{2}\|{\bf w} - {\bf w}'\|^2.
\end{align}
\end{assumption}
We suppose that the stochastic gradient $\nabla \ell({\bf w}(t))$ is unbiased, i.e., $\mathbb{E}[\nabla \ell({\bf w}(t))] = \nabla\mathcal{L}({\bf w}(t))$. However, since $\nabla \ell({\bf w}(t))$ and ${\bf m}(t)$ are not independent, the stochastic gradient $\nabla \ell({\bf w}(t)) \odot {\bf m}(t)$ is biased, that is, $\mathbb{E}[\nabla \ell({\bf w}(t))  \odot {\bf m}(t)] \neq \nabla\mathcal{L}({\bf w}(t))$. We make the following assumption for the stochastic gradient $\nabla \ell({\bf w}(t)) \odot {\bf m}(t)$.
\begin{assumption}[Bounded Variance]\label{ass:var}
We assume that the stochastic gradient $\nabla \ell({\bf w}) \odot {\bf m}(t) $ is biased and its variance is bounded. That is,
for any ${\bf w}(t)$ and ${\bf m}(t)$, we have
\begin{align}\label{ass:var:eqn:1}
\mathbb{E}[\nabla \ell({\bf w}(t)) \odot {\bf m}(t) ] = \nabla \mathcal{L}({\bf w}(t)) + b({\bf w}(t)),
\end{align}
and
\begin{align}\label{ass:var:eqn:2}
\mathbb{E}[\|\nabla \ell({\bf w}(t)) \odot {\bf m}(t) - \mathbb{E}[\nabla \ell({\bf w}(t)) \odot {\bf m}(t)]\|]^2 \le \nu \|\nabla\mathcal{L}({\bf w}(t)) + b({\bf w}(t))\|^2 + \sigma^2,
\end{align}
where $\sigma^2 \ge 0$ and $\nu \ge 0 $ are two constants.
\end{assumption}
\begin{assumption}[Mask-Incurred Error]\label{ass:error}
For any ${\bf w}(t)$ and ${\bf m}(t)$, we have
\begin{align}\label{ass:error:eqn}
\|\mathbb{E}[\nabla \ell({\bf w}(t)) \odot {\bf m}(t) -  \nabla \ell({\bf w}(t))] \| \le \delta \|\mathbb{E}[\nabla \ell({\bf w}(t))]\|,
\end{align}
where the constant $\delta \in [0,1]$.
\end{assumption}

We present the convergence property in the following theorem.
\begin{thm}[Formal, AMSS]
Under Assumptions~\ref{ass:smooth}, \ref{ass:var}, \ref{ass:error}, if the learning rate is set as $\eta = \frac{1}{(1+\nu)L\sqrt{T}}$, then
\begin{align}
\label{thm:1:eqn:0}
\frac{1}{T}\sum_{t=1}^{T}\|\nabla\mathcal{L}({\bf w}(t)) \|^2 \le   \frac{2(1+\nu)L\mathcal{L}\left({\bf w}(1)\right)}{\sqrt{T}(1-\delta^2)
}  + \frac{ \sigma^2}{(1+\nu)\sqrt{T}(1-\delta^2)}. 
\end{align}
\end{thm}

\begin{proof}
By Assumption~\ref{ass:smooth}, we have
\begin{align}\label{thm:1:eqn:1}
\nonumber \mathcal{L}\left({\bf w}(t+1)\right) - \mathcal{L}\left({\bf w}(t)\right) \le& \langle \nabla \mathcal{L}\left({\bf w}(t)\right), {\bf w}(t+1) - {\bf w}(t)\rangle + \frac{L}{2}\|{\bf w}(t+1) - {\bf w}(t)\|^2 \\
\overset{(\ref{app:eqn:updt:sgd})}{=}& -\eta \langle \nabla \mathcal{L}\left({\bf w}(t)\right), {\bf g}(t)\rangle + \frac{\eta^2L}{2}\| {\bf g}(t)\|^2 
\end{align}
where ${\bf g}(t) := \nabla\ell\left({\bf w}(t)\right) \odot {\bf m}(t)$.
Taking expectation over both sides of (\ref{thm:1:eqn:1}) and by using Assumption~\ref{ass:var}, we have
\begin{align}\label{thm:1:eqn:2}
\nonumber \mathbb{E}[\mathcal{L}\left({\bf w}(t+1)\right) - \mathcal{L}\left({\bf w}(t)\right)] \le& -\eta \langle \nabla \mathcal{L}\left({\bf w}(t)\right), \nabla \mathcal{L}\left({\bf w}(t)\right) + b({\bf w}(t))\rangle + \frac{\eta^2L}{2}\mathbb{E}[\| {\bf g}(t)\|^2] \\
\nonumber = & -\eta \langle \nabla \mathcal{L}\left({\bf w}(t)\right), \nabla \mathcal{L}\left({\bf w}(t)\right) + b({\bf w}(t))\rangle + \frac{\eta^2L}{2}(\mathbb{E}[\| {\bf g}(t) - \mathbb{E}[{\bf g}(t) ]\|^2] + \mathbb{E}[\|  \mathbb{E}[{\bf g}(t) ]\|^2] )\\
\nonumber \le & -\eta \langle \nabla \mathcal{L}\left({\bf w}(t)\right), \nabla \mathcal{L}\left({\bf w}(t)\right) + b({\bf w}(t))\rangle + \frac{\eta^2L}{2}\left( (1+\nu) \|\nabla\mathcal{L}({\bf w}(t)) + b({\bf w}(t))\|^2 + \sigma^2 \right)\\
\le & -\eta \langle \nabla \mathcal{L}\left({\bf w}(t)\right), \nabla \mathcal{L}\left({\bf w}(t)\right) + b({\bf w}(t))\rangle + \frac{\eta }{2} \|\nabla\mathcal{L}({\bf w}(t)) + b({\bf w}(t))\|^2 + \frac{\eta^2 L \sigma^2}{2}, 
\end{align}
where the last inequality is due to $\eta  \le \frac{1}{(1+\nu)L} $. Since $ - \langle a,b\rangle + \frac{\|b\|^2}{2} = \frac{\|a - b\|^2}{2}  - \frac{\|a\|^2}{2}$, then (\ref{thm:1:eqn:2}) implies that
\begin{align}\label{thm:1:eqn:3}
\nonumber \mathbb{E}[\mathcal{L}\left({\bf w}(t+1)\right) - \mathcal{L}\left({\bf w}(t)\right)] \le& -\eta \langle \nabla \mathcal{L}\left({\bf w}(t)\right), \nabla \mathcal{L}\left({\bf w}(t)\right) + b({\bf w}(t))\rangle + \frac{\eta^2L}{2}\mathbb{E}[\| {\bf g}(t)\|^2] \\
\le &   \frac{\eta }{2} \| b({\bf w}(t))\|^2 -  \frac{\eta }{2} \|\nabla\mathcal{L}({\bf w}(t)) \|^2    + \frac{\eta^2 L \sigma^2}{2}, 
\end{align}
Next, by (\ref{ass:var:eqn:1}) in Assumption~\ref{ass:var} and (\ref{ass:error:eqn}) in Assumption~\ref{ass:error}, we know
\begin{align}\label{thm:1:eqn:4}
\|b({\bf w}(t))\| = \|\mathbb{E}[\nabla \ell({\bf w}(t)) \odot {\bf m}(t) -  \nabla \ell({\bf w}(t))] \| \le \delta \|\mathbb{E}[\nabla \ell({\bf w}(t))]\| = \delta \|\nabla\mathcal{L}({\bf w}(t))\|.
\end{align}
Therefore, by (\ref{thm:1:eqn:3}) and (\ref{thm:1:eqn:4}) we have
\begin{align}\label{thm:1:eqn:5}
\mathbb{E}[\mathcal{L}\left({\bf w}(t+1)\right) - \mathcal{L}\left({\bf w}(t)\right)] 
\le &  -  \frac{\eta (1-\delta^2) }{2} \|\nabla\mathcal{L}({\bf w}(t)) \|^2    + \frac{\eta^2 L \sigma^2}{2}, 
\end{align}
which implies
\begin{align}
\|\nabla\mathcal{L}({\bf w}(t)) \|^2 
\le &  \frac{2\mathbb{E}[\mathcal{L}\left({\bf w}(t)\right) - \mathcal{L}\left({\bf w}(t+1)\right)]}{\eta(1-\delta^2)
}  + \frac{\eta L \sigma^2}{1-\delta^2}.
\end{align}
By summing up for $t=1,\dots,T$, we have
\begin{align}
\frac{1}{T}\sum_{t=1}^{T}\|\nabla\mathcal{L}({\bf w}(t)) \|^2 
\le \frac{2\mathbb{E}[\mathcal{L}\left({\bf w}(1)\right) - \mathcal{L}\left({\bf w}(T+1)\right)]}{T\eta(1-\delta^2)
}  + \frac{\eta L \sigma^2}{1-\delta^2}
\le \frac{2\mathcal{L}\left({\bf w}(1)\right)}{T\eta(1-\delta^2)
}  + \frac{\eta L \sigma^2}{1-\delta^2}.
\end{align}
By setting $\eta = \frac{1}{(1+\nu)L\sqrt{T}}$, we get
\begin{align}\label{thm:1:eqn:7}
\frac{1}{T}\sum_{t=1}^{T}\|\nabla\mathcal{L}({\bf w}(t)) \|^2 \le   \frac{2(1+\nu)L\mathcal{L}\left({\bf w}(1)\right)}{\sqrt{T}(1-\delta^2)
}  + \frac{ \sigma^2}{(1+\nu)\sqrt{T}(1-\delta^2)}.
\end{align}
\end{proof}

\subsection{Case II: Unbiased Stochastic Gradient}
For the second case of unbiased stochastic gradient, recall that the update step of SGD is 
\begin{align}\label{app:eqn:updt:sgd:unbiased}
{\bf w}(t+1)={\bf w}(t)-\eta \nabla \ell({\bf w}(t)) \odot \hat{\bf m}(t)
\end{align}
where $\nabla \ell({\bf w}(t))$ is the stochastic version of the gradient of loss function $\nabla\mathcal{L}({\bf w}(t))$ at ${\bf w}(t)$ and $\eta>0$ is the learning rate, and the element of $\hat{\bf m}(t)$ is given by
\begin{equation}
\hat{m}_j(t)= \left\{\begin{array}{ll}
\frac{1}{p_j(t)},  & \text{if}\;\;w_j(t) \in {{\mathcal{S}}(t)}, \\
0,  & \text{otherwise},
\end{array}\right.
\end{equation}
with $p_i(t)=\mathbb{P}( {w}_i(t)\in {s_i(t)})$.
We suppose that the stochastic gradient $\nabla \ell({\bf w}(t))$ is unbiased, i.e., $\mathbb{E}[\nabla \ell({\bf w}(t))] = \nabla\mathcal{L}({\bf w}(t))$. Then we have
$ \mathbb{E}[\nabla \ell({\bf w}(t))  \odot \hat{\bf m}(t)] = \mathbb{E}[\nabla \ell({\bf w}(t)) \odot  {{\bf p}(t)}^{-1} \odot {\bf m}(t)  ] =  \mathbb{E}[\mathbb{E}_{{\bf m}(t)}[\nabla \ell({\bf w}(t)) \odot  {{\bf p}(t)}^{-1} \odot {\bf m}(t) |\nabla \ell({\bf w}(t))] ]   = \mathbb{E}[\nabla \ell({\bf w}(t)) \odot  {{\bf p}(t)}^{-1} \odot\mathbb{E}_{{\bf m}(t)}[ {\bf m}(t) |\nabla \ell({\bf w}(t))] ] =\mathbb{E}[\nabla \ell({\bf w}(t)) ]= \nabla\mathcal{L}({\bf w}(t))$, indicating that $\nabla \ell({\bf w}(t))  \odot \hat{\bf m}(t)$ is also unbiased. We make the following assumption for the stochastic gradient $\nabla \ell({\bf w}(t)) \odot \hat{\bf m}(t)$.
\begin{assumption}[Bounded Variance]\label{ass:var:unbiased}
We assume that the stochastic gradient $\nabla \ell({\bf w}) \odot \hat{\bf m}(t) $ is biased and its variance is bounded. That is,
for any ${\bf w}(t)$ and $\hat{\bf m}(t)$, we have
\begin{align}\label{ass:var:unbiased:eqn:1}
\mathbb{E}[\nabla \ell({\bf w}(t)) \odot \hat{\bf m}(t) ] = \nabla \mathcal{L}({\bf w}(t)),
\end{align}
and
\begin{align}\label{ass:var:unbiased:eqn:2}
\mathbb{E}[\|\nabla \ell({\bf w}(t)) \odot \hat{\bf m}(t) - \mathbb{E}[\nabla \ell({\bf w}(t)) \odot \hat{\bf m}(t)]\|]^2 \le \nu \|\nabla\mathcal{L}({\bf w}(t))\|^2 + \sigma^2,
\end{align}
where $\sigma^2 \ge 0$ and $\nu \ge 0 $ are two constants.
\end{assumption}

We present the convergence property in the following theorem.
\begin{thm}[Formal, AMSS+]
Under Assumptions~\ref{ass:smooth}, \ref{ass:var:unbiased}, if the learning rate is set as $\eta = \frac{1}{(1+\nu)L\sqrt{T}}$, then
\begin{align}\label{thm:2:eqn:0}
\frac{1}{T}\sum_{t=1}^{T}\|\nabla\mathcal{L}({\bf w}(t)) \|^2 \le   \frac{2(1+\nu)L\mathcal{L}\left({\bf w}(1)\right)}{\sqrt{T}
}  + \frac{ \sigma^2}{(1+\nu)\sqrt{T}}. 
\end{align}
\end{thm}

\begin{proof}
By Assumption~\ref{ass:smooth}, we have
\begin{align}\label{thm:2:eqn:1}
\nonumber \mathcal{L}\left({\bf w}(t+1)\right) - \mathcal{L}\left({\bf w}(t)\right) \le& \langle \nabla \mathcal{L}\left({\bf w}(t)\right), {\bf w}(t+1) - {\bf w}(t)\rangle + \frac{L}{2}\|{\bf w}(t+1) - {\bf w}(t)\|^2 \\
\overset{(\ref{app:eqn:updt:sgd:unbiased})}{=}& -\eta \langle \nabla \mathcal{L}\left({\bf w}(t)\right), \hat{\bf g}(t)\rangle + \frac{\eta^2L}{2}\| \hat{\bf g}(t)\|^2 
\end{align}
where $\hat{\bf g}(t) := \nabla\ell\left({\bf w}(t)\right) \odot \hat{\bf m}(t)$.
Taking expectation over both sides of (\ref{thm:2:eqn:1}) and by using Assumption~\ref{ass:var:unbiased}, we have
\begin{align}\label{thm:2:eqn:2}
\nonumber \mathbb{E}[\mathcal{L}\left({\bf w}(t+1)\right) - \mathcal{L}\left({\bf w}(t)\right)] \le& -\eta \|\nabla \mathcal{L}\left({\bf w}(t)\right)\|^2+ \frac{\eta^2L}{2}\mathbb{E}[\| \hat{\bf g}(t)\|^2] \\
\nonumber = & -\eta \|\nabla \mathcal{L}\left({\bf w}(t)\right)\|^2 + \frac{\eta^2L}{2}(\mathbb{E}[\| \hat{\bf g}(t) -\nabla \mathcal{L}\left({\bf w}(t)\right)\|^2] + \|\nabla \mathcal{L}\left({\bf w}(t)\right)\|^2 )\\
\nonumber \le & -\eta \|\nabla \mathcal{L}\left({\bf w}(t)\right)\|^2 + \frac{\eta^2L}{2}\left( (1+\nu) \|\nabla\mathcal{L}({\bf w}(t)) \|^2 + \sigma^2 \right)\\
\le & -\eta \|\nabla \mathcal{L}\left({\bf w}(t)\right)\|^2 + \frac{\eta }{2} \|\nabla\mathcal{L}({\bf w}(t)) \|^2 + \frac{\eta^2 L \sigma^2}{2}, 
\end{align}
where the last inequality is due to $\eta  \le \frac{1}{(1+\nu)L} $. 
Therefore, we have
\begin{align}
\|\nabla\mathcal{L}({\bf w}(t)) \|^2 
\le &  \frac{2\mathbb{E}[\mathcal{L}\left({\bf w}(t)\right) - \mathcal{L}\left({\bf w}(t+1)\right)]}{\eta
}  + \eta L \sigma^2.
\end{align}
By summing up for $t=1,\dots,T$, we have
\begin{align}
\frac{1}{T}\sum_{t=1}^{T}\|\nabla\mathcal{L}({\bf w}(t)) \|^2 
\le \frac{2\mathbb{E}[\mathcal{L}\left({\bf w}(1)\right) - \mathcal{L}\left({\bf w}(T+1)\right)]}{T\eta
}  + \eta L \sigma^2
\le \frac{2\mathcal{L}\left({\bf w}(1)\right)}{T\eta
}  +\eta L \sigma^2.
\end{align}
By setting $\eta = \frac{1}{(1+\nu)L\sqrt{T}}$, we get
\begin{align}\label{thm:2:eqn:7}
\frac{1}{T}\sum_{t=1}^{T}\|\nabla\mathcal{L}({\bf w}(t)) \|^2 \le   \frac{2(1+\nu)L\mathcal{L}\left({\bf w}(1)\right)}{\sqrt{T}
}  + \frac{ \sigma^2}{(1+\nu)\sqrt{T}}.
\end{align}
\end{proof}

\section{Experiments}

\begin{figure*}[tbp]
\centering
\subfloat[CREMA-D]{
  \includegraphics[width=0.24\textwidth]{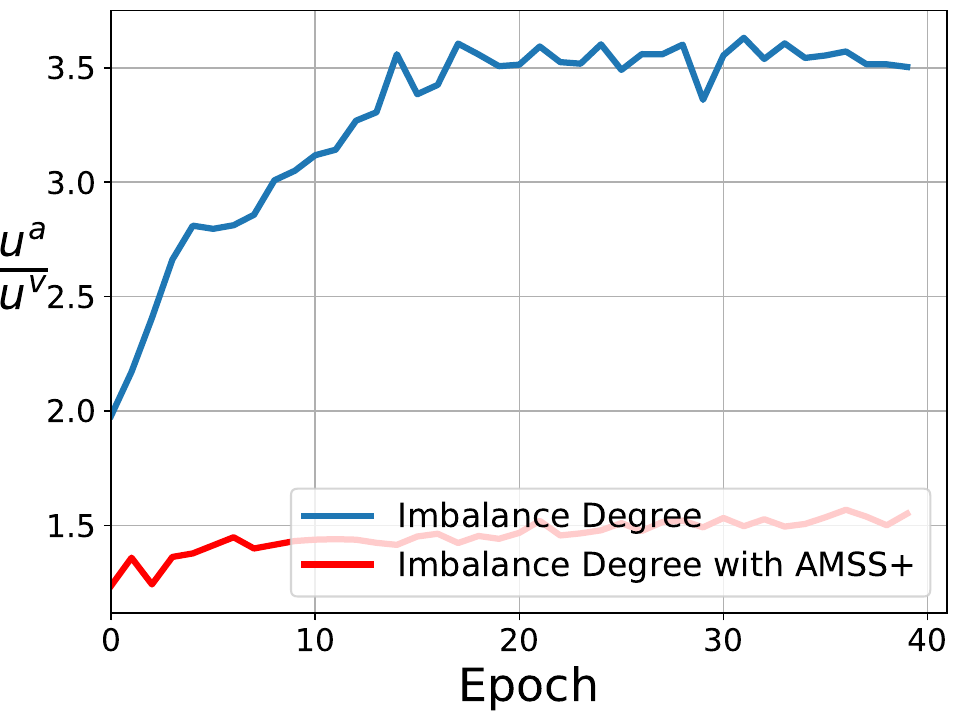}\hfill
  \includegraphics[width=0.24\textwidth]{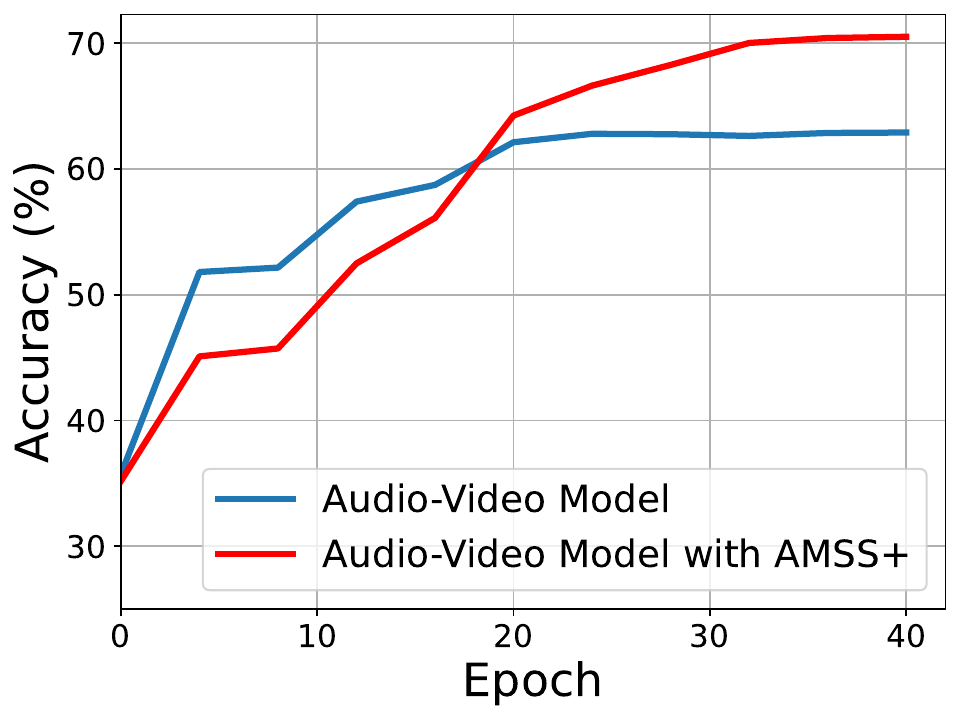}\hfill
  \includegraphics[width=0.24\textwidth]{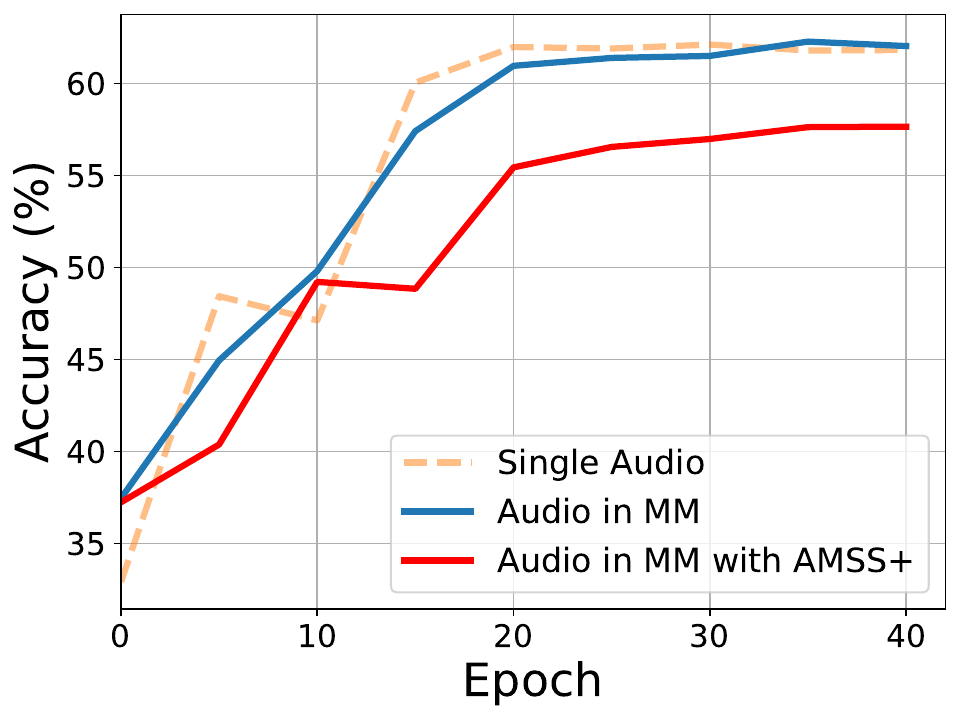}\hfill
  \includegraphics[width=0.24\textwidth]{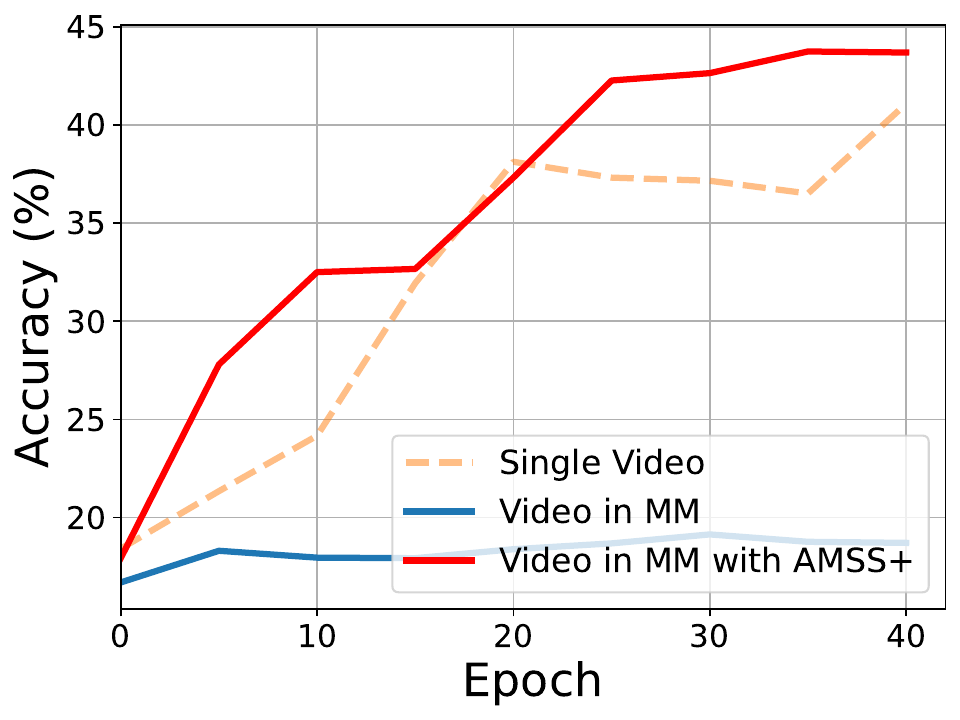}\hfill
  } \\
\subfloat[Twitter-15]{
  \includegraphics[width=0.24\textwidth]{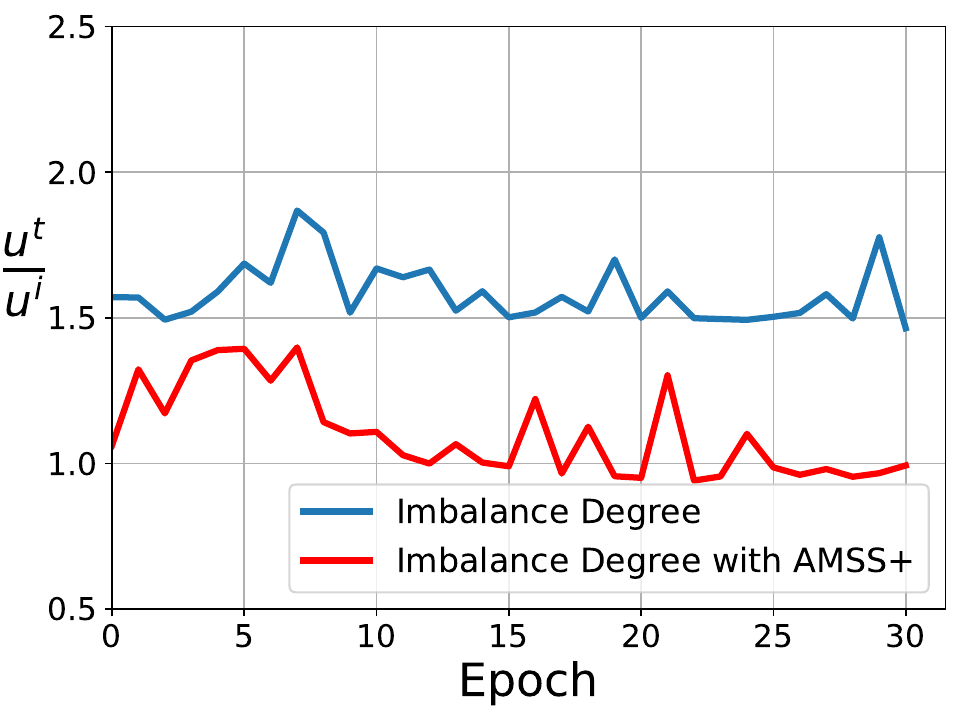}\hfill
  \includegraphics[width=0.24\textwidth]{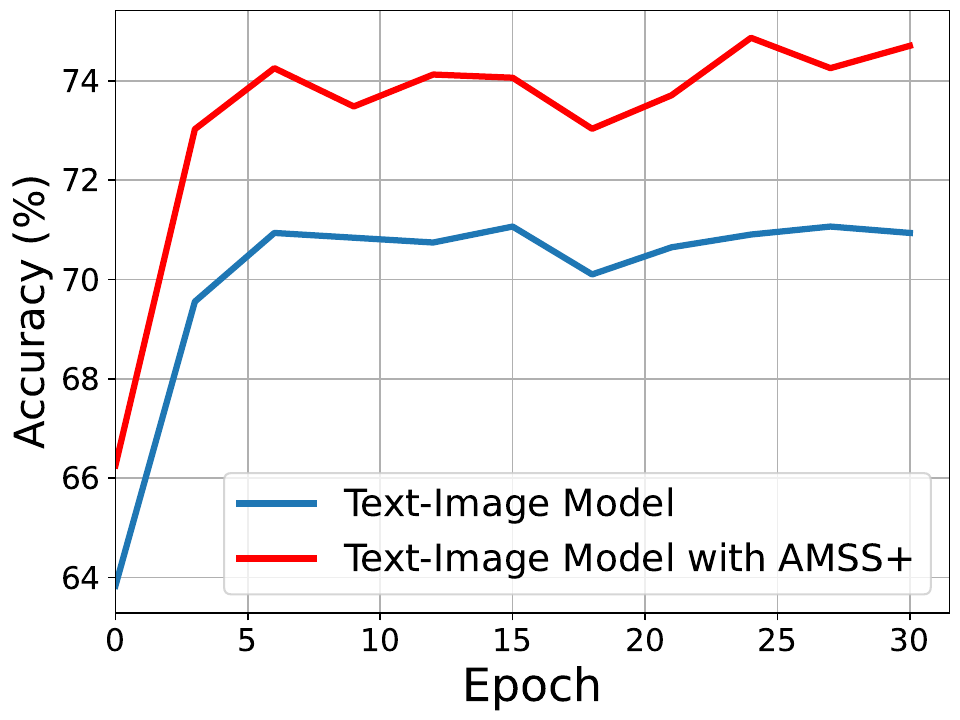}\hfill
  \includegraphics[width=0.24\textwidth]{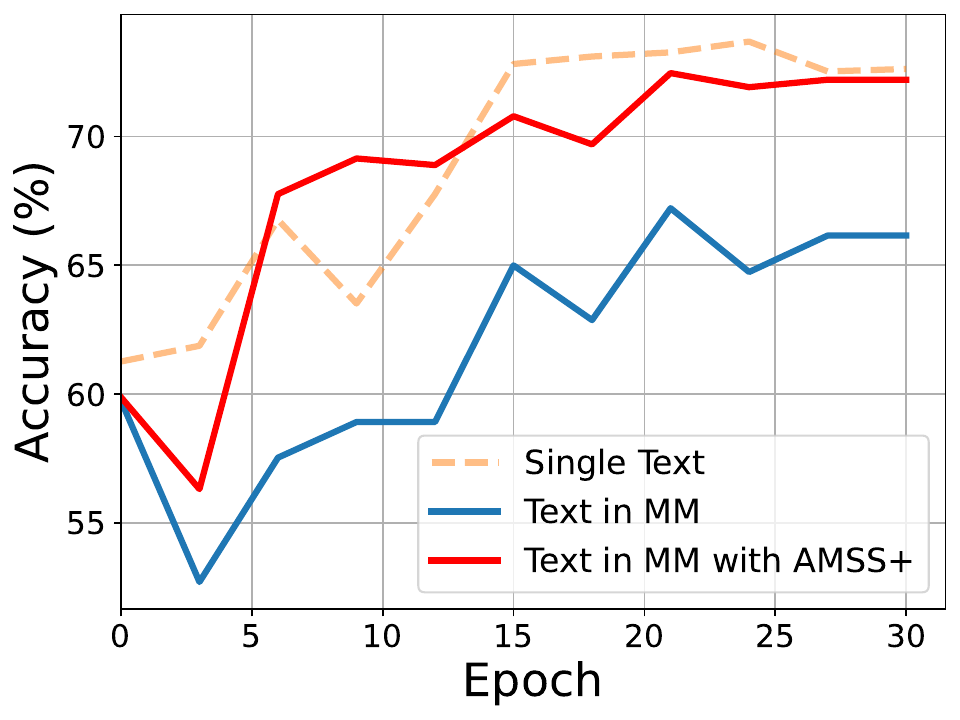}\hfill
  \includegraphics[width=0.24\textwidth]{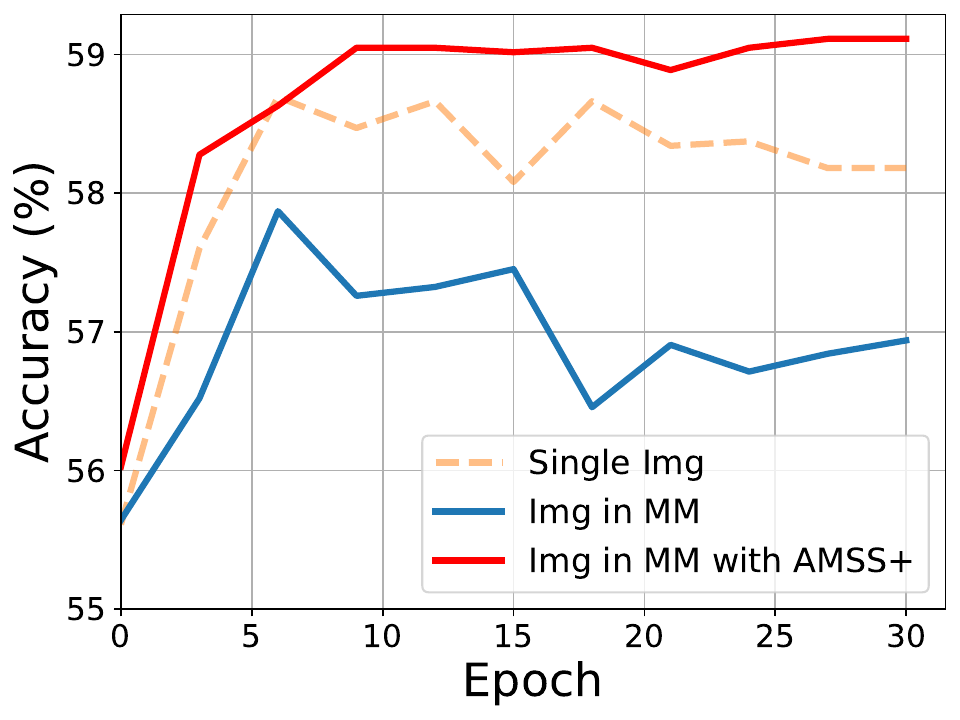}\hfill
  }
  \caption{The same as Figure~\ref{fig:imbalance}, but for CREMA-D and Twitter-15 datasets.}
  \label{fig:imbalance_supp}
  
\end{figure*}

Figure~\ref{fig:imbalance_supp} illustrates the analysis of modality imbalance in two additional datasets (CREMA-D, Twitter-15). Similar to the results depicted in Figure~\ref{fig:imbalance}, the modality imbalance degree on both datasets exhibits a reduction through the AMSS+ strategy.
On the CREMA-D dataset, our method initially performs less favorably than the Baseline method but exhibits significant improvement in later stages compared to the Baseline. Analysis of the curve of imbalance degree variations indicates that early application of the AMSS+ modulation strategy mitigates the dominance of the Audio modality optimization, enabling the model to better explore information from Video modality and ultimately achieve superior performance. 
Furthermore, on the Twitter15 dataset, the balancing strategy implemented by AMSS+ results in improved performance across both modality branches and overall model performance. This improvement is characterized by achieving performance levels for individual modality branches in multimodal models that are comparable to those achieved through separate training of single modality models. Additional experiments once again validate the effectiveness and reliability of AMSS+.

\end{document}